%% file: main.tex
\documentclass{siamart250211}

\usepackage{graphicx}
\graphicspath{{images/}}
\usepackage{amssymb,amsfonts,bm}
\usepackage{subfigure}
\usepackage[algo2e]{algorithm2e}
\usepackage{multirow}
\usepackage{booktabs}

\newtheorem{remark}[theorem]{Remark}

\title{A Globally Convergent Algorithm for Total
Scaled-Gradient Variation via Cone-Constrained
Bilinear Decomposition}
\author{Haibin Su\thanks{Yau Mathematical Sciences Center, Tsinghua University, Beijing 100084, China (\email{hbsu@mail.tsinghua.edu.cn})}
\and Chunlin Wu\thanks{School of Mathematical Sciences, Nankai University, Tianjin, 300071, China  (\email{wucl@nankai.edu.cn})}
\and Huibin Chang\thanks{School of Mathematical Sciences and Institute of Mathematics and Interdisciplinary Sciences, Tianjin Normal University, Tianjin, 300387, China (\email{changhuibin@tjnu.edu.cn})}
\and Zhifang Liu\thanks{Corresponding author. School of Mathematical Sciences and Institute of Mathematics and Interdisciplinary Sciences, Tianjin Normal University, Tianjin, 300387, China (\email{matlzhf@tjnu.edu.cn})}
}

\begin{document}

\maketitle

\begin{abstract}
The total scaled-gradient variation (TSGV) regularizer, derived from sparse modeling of piecewise-linear structures, has been shown to preserve edges and corners in image restoration. However, its highly nonconvex and nonlinear nature poses severe computational challenges, as existing methods often suffer from parameter sensitivity or lack convergence guarantees. To overcome this, we propose a tailored bilinear decomposition that decouples the nonlinear weighted gradient in the TSGV regularizer. This approach yields an equivalent optimization problem governed by cone or sphere constraints, depending on the chosen scaling function. In particular, the cone constraint plays a central role in characterizing edge- and corner-preserving behavior. We solve this reformulation using the alternating minimization method (AMM) equipped with a majorization--minimization strategy, ensuring a monotonic decrease in energy without step-size tuning. Furthermore, we provide a geometric interpretation of the edge-preserving properties of these constraints by analyzing their asymptotic behavior near image singularities. We establish the global convergence of the proposed method to a critical point within the Kurdyka--Łojasiewicz framework. Extensive numerical experiments on Gaussian denoising and non-line-of-sight (NLOS) imaging show that the proposed method achieves PSNR and SSIM competitive with or superior to representative variational methods, especially at high noise levels, and improves the structural reconstruction under dense and sparse scanning.

\end{abstract}

\begin{keywords}
bilinear
decomposition, majorization-minimization method, cone constraint, total scaled-gradient variation, non-line-of-sight imaging
\end{keywords}

\begin{AMS}
65K10, 68U10, 94A08
\end{AMS}

\input{sections/sec-introduction-new}
\input{sections/sec-reformulation}

\input{sections/sec-asymptotic-comparison}
\input{sections/sec-mmamm-method}
\input{sections/sec-convergence-analysis}
\input{sections/sec-experiments}
\input{sections/sec-conclusion}

\section*{Acknowledgments}

We sincerely thank the authors who made their MATLAB codes publicly available online.

\bibliographystyle{siamplain}
\bibliography{references}
\end{document}

%% file: sections/sec-introduction-new.tex

\section{Introduction}

High-order variational models \cite{KB2010,CBL2016,AC1997,ML2003,Shen2003euler,yang2025mixed,QZ2022,zhu2020image,Zhu2013Image} provide effective High-order or geometric priors in image restoration, particularly for preserving sharp edges and corner contrasts \cite{CBL2016,wu2024variational,QZ2022, WZ2012}. Notably, the recently proposed total scaled-gradient variation (TSGV) \cite{wu2024variational} serves as a general high-order regularization framework that bridges geometry-driven approaches with sparse representation. By using sparse modeling of piecewise-linear structures, TSGV incorporates several classical high-order regularizers as special cases and provides theoretical guarantees for structural preservation. 

In this paper, we develop an efficient numerical method for solving the TSGV model \cite{wu2024variational}, whose continuous energy functional is formulated as
\begin{equation}\label{TSGV}
    \min_{u} \int_{\Omega} \|\nabla (\psi(|\nabla u|) \nabla u)\|_\mathrm{F} \mathrm{d}x\mathrm{d} y + \frac{\lambda}{2} \int_{\Omega}(u-f)^2 \mathrm{d}x\mathrm{d} y,
\end{equation}
where $f$ is an observed noisy image defined on $\Omega$, and $\lambda>0$ is a regularization parameter balancing regularity and data fidelity. Here, $\nabla u$ denotes the gradient of the scalar field $u$, the outer $\nabla$ applied to the vector field $\psi(|\nabla u|)\nabla u$ denotes its Jacobian matrix, and $\|\cdot\|_\mathrm{F}$ is the Frobenius norm, given for a $2\times2$ matrix $H=[H_{ij}]_{i,j=1}^{2}$ by $\|H\|_\mathrm{F}=\big(\sum_{i,j}H_{ij}^{2}\big)^{1/2}$. The scaling function $\psi:[0,+\infty) \to (0,\infty)$ is a monotonically decreasing $C^1$ function.
We specifically focus on the scaling function 
defined by $\psi(s)=\frac{1}{(b+s^j)^{1/j}}$ with $b>0$ and integer $j\geq1$. 

The primary motivation for investigating the TSGV model stems from its capability to provide a unified variational framework derived from sparse modeling of piecewise-linear structures \cite{wu2024variational}. By modulating the scaling function $\psi$, this framework encompasses several representative high-order models, including the Laplacian regularization model with $\phi(s) = s$ \cite{YLY2000}, the Lysaker, Lundervold, and Tai (LLT) model \cite{ML2003}, the Mean curvature (MC) model \cite{WZ2012} and Weingarten map-based model \cite{QZ2022}.
Geometrically, TSGV effectively mitigates the staircase artifacts \cite{AC1997,dobson1996recovery}, contrast reduction \cite{deledalle2015debiasing,YM2002}, and corner smearing \cite{bellettini2002total} inherent in piecewise-constant regularization while maintaining rigorously proven capabilities for edge and corner contrast preservation \cite{wu2024variational}. However, these geometric advantages are accompanied by significant mathematical complexity. The nested nonlinearity and non-convexity induced by the coupled high-order term $\nabla(\psi(|\nabla u|)\nabla u)$ pose considerable challenges for numerical minimization, necessitating the development of stable and theoretically guaranteed optimization algorithms.

To address these difficulties, augmented Lagrangian methods (ALM) have been widely used for nonconvex high-order variational models \cite{CBL2016,Liu2021Variable,XT2011,MY2016,QZ2022,WZ2013, Zhu2013Image,Duan2013Afast,Zhang2017Fast,liu2019proximal}, as they decompose the problem into more tractable subproblems. However, for highly nonconvex geometric models, ALM-based schemes inevitably introduce multiple penalty parameters, making their numerical stability highly sensitive to heuristic tuning. Alternatively, operator-splitting (OS) methods based on Lie scheme and Marchuk-Yanenko discretization \cite{LD2019,glowinski2019fast,he2025euler,HL2022,lu2026total,wu2024variational} offer a simpler algorithmic structure with fewer parameters. Yet, a major limitation of these methods is the lack of a rigorous convergence theory for nonconvex high-order models. More recently, bilinear decomposition approaches \cite{liu2024fast,liu2025minimizing} were proposed to establish theoretical convergence within the Kurdyka–Łojasiewicz (KL) framework \cite{bolte2014proximal}. These approaches were successfully implemented via a hybrid alternating method (HALM) for the Euler's elastica (EE) model \cite{liu2024fast} and a proximal alternating direction method of multipliers (ADMM) for the MC model \cite{liu2025minimizing}. Nevertheless, their practical performance still depends on parameter selection. Specifically, the subproblem solvers in HALM rely on direct gradient-based updates that require explicit step sizes \cite{liu2024fast}, whereas the proximal ADMM framework inevitably introduces auxiliary proximal and penalty parameters to guarantee convergence \cite{liu2025minimizing}. Tuning these algorithmic parameters—whether explicit step sizes or penalty weights—can affect the convergence speed and numerical stability of these algorithms.

To overcome these limitations, we introduce a bilinear decomposition tailored to a class of models with nested nonlinear structures. The TSGV model is used as a representative example to illustrate the proposed framework. While inspired by recent applications of bilinear decomposition to the Euler elastica \cite{liu2024fast} and mean curvature \cite{liu2025minimizing} models, our approach addresses a different mathematical structure. Existing methods primarily separate the standard two-dimensional image gradient into its magnitude and normalized direction \cite{liu2024fast}, or extend this concept to a three-dimensional image surface gradient \cite{liu2025minimizing}. The TSGV framework, however, requires decoupling the highly nonlinear and coupled vector field $\psi(|\nabla u|)\nabla u$. We achieve this exact decoupling by introducing an augmented direction field whose third component explicitly encodes the scaling relation. This variable transformation reformulates the unconstrained problem into an equivalent constrained formulation governed by non-Riemannian cone or spherical constraints. To efficiently solve the resulting constrained model without the explicit step-size dependencies of existing HALM algorithms \cite{liu2024fast}, we develop a majorization-minimization alternating minimization method (MMAMM). By constructing a surrogate function, the MMAMM scheme ensures monotonic energy descent, avoiding the step-size sensitivity and potential numerical instability associated with traditional gradient-based solvers.

Our main contributions are summarized as follows:
\begin{itemize}
\item We derive an equivalent constrained bilinear reformulation of the TSGV model. In contrast to the decompositions previously developed for EE and MC regularization, the proposed formulation explicitly represents the nonlinear scaling relation in $\psi(|\nabla u|)\nabla u$ through an augmented geometric variable.

\item We provide a geometric interpretation of the constrained feasibility sets (e.g., cone or spherical constraints) generated by different scaling functions. By analyzing their asymptotic behavior near image edges, we explain their distinct edge-preserving properties, which are further corroborated by numerical experiments.

\item To handle the nonconvex subproblems generated by the decomposition, we incorporate an majorization-minimization strategy and construct a suitable surrogate function for the penalized discrete model. This scheme ensures a monotonic decrease in energy across iterations and avoids several practical difficulties of direct gradient-based methods, such as slow convergence and sensitivity to step-size selection. Moreover, we establish the sequence boundedness and convergence to a critical point within the KL framework, thereby providing the rigorous mathematical guarantees often absent in traditional operator-splitting approaches.

\item The proposed method is validated through numerical experiments on Gaussian denoising and NLOS imaging. It achieves competitive or improved reconstruction quality compared with representative variational methods, especially in preserving structural details under high noise levels and both dense and sparse NLOS scanning settings. The extension to NLOS imaging also illustrates that the proposed numerical framework is not restricted to two-dimensional denoising but can be applied to more general large-scale inverse problems.
\end{itemize}

The remainder of this paper is organized as follows. Section \ref{sec:reformulation} establishes the necessary mathematical preliminaries. Section \ref{sec:asymptotic comparison} analyzes the asymptotic behavior of the cone and other constraints near edges. Section \ref{sec:MMAMM} presents the discrete bilinear decomposition formulation of model \eqref{TSGV} and introduces the MM-based algorithm for the resulting subproblems. Section \ref{sec:convergence} provides the convergence analysis of the proposed algorithm. Numerical experiments on image denoising and NLOS imaging are reported in Section \ref{sec:experiments}. Finally, Section \ref{sec:conclusion} concludes the paper and discusses possible directions for future research.

%% file: sections/sec-reformulation.tex

\section{A bilinear decomposition method for the TSGV-based model} \label{sec:reformulation}
This section presents an equivalent reformulation of the TSGV model \eqref{TSGV}. To separate the nonlinearities inherent in the original formulation, we introduce two variables that decouple the weighted gradient structure. In particular, we consider the scaling function
\[
\psi(s)=\psi_j(s)=\frac{1}{(b+s^j)^{1/j}},\quad b>0,\quad j\geq 1.
\]
Let
\[
q = \frac{1}{\psi_j(|\nabla u|)},\quad
\beta = \frac{1}{\psi_j(0)}=b^{1/j}>0,
\]
and
\[
\vec{n} = [n_1,n_2,n_3]^{\top} := [\psi_j(|\nabla u|) \nabla_x u, \psi_j(|\nabla u|) \nabla_y u, \psi_j(|\nabla u|) \beta]^{\top}.
\]
With these definitions, the following relations hold:
\[
\left[\begin{array}{c}
\nabla u \\ \beta
\end{array}\right]
= q \vec{n},\quad
q \geq \beta, \quad
|\vec{n}|_{\psi_j} = 1, \quad
n_3\geq0,
\]
where $|\cdot|_{\psi_j}$ denotes a length-like quantity induced by the scaling function $\psi_j$. More precisely, we have
\[
\left(\sqrt{n_1^2+n_2^2}\right)^j+n_3^j
=
\frac{|\nabla u|^j}{\beta^j+|\nabla u|^j}
+
\frac{\beta^j}{\beta^j+|\nabla u|^j}
=1.
\]
Therefore, for the family $\psi_j$, the induced length-like quantity is defined by
\[
|\vec{n}|_{\psi_j}
=
\left[
\left(\sqrt{n_1^2+n_2^2}\right)^j+n_3^j
\right]^{1/j}.
\]
For $\psi(s) = \psi_1(s) = \frac{1}{b+s}$ and $\beta = b$, the constraint becomes
\[
|\vec{n}|_{\psi_1} = \sqrt{n_1^2 + n_2^2} + n_3 = 1,
\]
which is a cone constraint. For $\psi(s) = \psi_2(s) = \frac{1}{\sqrt{b+s^2}}$ and $\beta = \sqrt{b}$, the constraint reduces to the Euclidean norm
\[
|\vec{n}|_{\psi_2} = \sqrt{n_1^2 + n_2^2 + n_3^2} = 1.
\]

Using the above variable transformations, the original problem \eqref{TSGV} can be reformulated as the following constrained minimization problem:
\begin{equation}\label{bilinear_model}
\begin{aligned} &\min_{u,\vec{n},q} \int_{\Omega} \left\| \nabla \left[\begin{array}{c} n_1 \\ n_2 \end{array}\right] \right\|_\mathrm{F} \mathrm{d}x\mathrm{d} y + \frac{\lambda}{2} \int_{\Omega}(u-f)^2 \mathrm{d}x\mathrm{d} y\\ &\ \ \mbox{s.t.} \ \left[\begin{array}{c} \nabla u \\ \beta \end{array}\right] = q \vec{n}, \ q \geq \beta, \ |\vec{n}|_{\psi_j} = 1,\ n_3\geq0.
\end{aligned}
\end{equation}
In this paper, we focus on two representative choices, $j=1$ and $j=2$, corresponding to the cone constrained and upper-hemisphere constrained formulations, respectively.

%% file: sections/sec-asymptotic-comparison.tex
\section{Asymptotic analysis of the different constraints near edges} \label{sec:asymptotic comparison}
In this section, we compare the different constraints in the edge regime. Denote
\[
p=\nabla u=\begin{bmatrix}
u_x\
u_y
\end{bmatrix},\quad
s=|\nabla u|=\sqrt{u_x^2+u_y^2},
\]
where $p$ is the image gradient and $s$ is its magnitude. Let $n_T=(n_1,n_2)^T$ be the tangential part of the normal field $n$. From the constraints in the proposed model, we have
\[
p=q n_T,\quad
\beta=qn_3.
\]
For a different scaling function
\[
\psi_j(s)=\frac{1}{(\beta^j+s^j)^{1/j}},\quad j\geq 1,
\]
we obtain
\[
n_T^j=\frac{p}{(\beta^j+s^j)^{1/j}},
\quad
n_3^j=\frac{\beta}{(\beta^j+s^j)^{1/j}}.
\]
The superscript $j$ refers to the constraint induced by the scaling function $\psi_j$. For $s>0$, define the unit gradient direction by
\[
e=\frac{p}{|p|}=\frac{\nabla u}{|\nabla u|}.
\]
The tangential normal component can be written in the unified form
\[
n_T^j=\rho^j(s)e,
\]
where
\[
\rho^j(s)=\frac{s}{(\beta^j+s^j)^{1/j}}.
\]

The regularization term in the proposed model penalizes the spatial variation of $n_T$. Hence, the difference among the different constraints can be understood by comparing the derivatives of $n_T^j$, which is
\[
\nabla n_T^j = (\rho^j)'(s)e\otimes\nabla s + \rho^j(s)\nabla e.
\]
Since $|e|=1$, we have $e\cdot e_x=0$ and $e\cdot e_y=0$. Consequently,
\begin{equation}\label{different_form_of_regular}
\|\nabla n_T^j\|_\mathrm{F}^2 = \big((\rho^j)'(s)\big)^2|\nabla s|^2 + (\rho^j)^2(s)\|\nabla e\|_\mathrm{F}^2.
\end{equation}
This identity shows that the regularization penalizes two different geometric quantities. The first term in \eqref{different_form_of_regular} penalizes variation of the gradient magnitude, while the second term in \eqref{different_form_of_regular} penalizes variation of the gradient direction.

We now compare the two coefficients in the large-gradient regime.
As $s\to\infty$, the directional coefficient satisfies
\[
(\rho^j)^2(s) = \left(1+\frac{\beta^j}{s^j}\right)^{-2/j}
=1-\frac{2}{j}\frac{\beta^j}{s^j}+O(s^{-2j}).
\]
In particular, the cone coefficient satisfies
\[
(\rho^1)^2(s) = \left(1+\frac{\beta}{s}\right)^{-2} = 1-\frac{2\beta}{s} +O(s^{-2}),
\]
and for $j\geq2$,
\[
(\rho^j)^2(s)-(\rho^1)^2(s) = \frac{2\beta}{s}+O(s^{-j})>0, \qquad s\to\infty.
\]
It follows that, in the edge regime $s\to\infty$, the cone constraint assigns a smaller weight to the directional variation term $\|\nabla e\|_\mathrm{F}^2$. Therefore, changes in the normal direction are less penalized under the cone constraint.

Next, we compare the weights associated with variations of the gradient magnitude. Direct differentiation with respect to $s$ gives
\begin{align*}
((\rho^j)'(s))^2=\frac{\beta^{2j}}{(\beta^j+s^j)^{2/j+1}} &=
\beta^{2j} s^{-2j-2}
\left(1+\frac{\beta^j}{s^j}\right)^{-2\left(1+\frac{1}{j}\right)} \\
&= \beta^{2j} s^{-2j-2}
+O(s^{-3j-2}),
\end{align*}
as $s\to\infty$. For $j\geq2$,
\[
((\rho^1)'(s))^2-((\rho^j)'(s))^2=\frac{\beta^2}{s^4}+O(s^{-5})>0.
\]
Thus, while the cone constraint penalizes directional variations more weakly, it penalizes variations of the gradient magnitude more strongly in the edge regime. This shows that the cone constraint is anisotropically adapted to edge geometry: it allows the edge direction to change sharply, while suppressing oscillations in the edge regime.

Although the comparison is made through $\|\nabla n_T^j\|_{\mathrm F}^2$, it only serves to identify the orthogonal components inside the Frobenius norm. Taking the square root preserves the ordering of the corresponding coefficients, and hence the same interpretation applies to the regularizer $\|\nabla n_T^j\|_{\mathrm F}$.



%% file: sections/sec-mmamm-method.tex
\section{Majorization-minimization alternating minimization method} \label{sec:MMAMM}
In this section, we introduce an alternating minimization method based on the MM strategy, referred to as MMAMM. The MM strategy \cite{hunter2004tutorial, sun2016majorization} is employed to handle the nonconvex $\mathbf{n}$-subproblem by constructing a tractable surrogate function that majorizes the original objective at the current iterate, guaranteeing monotone energy descent without requiring step size tuning. We first present the discretization of the proposed model. Then, we describe the MM-based update for the $\mathbf{n}$-subproblem and discuss implementation of the method under different geometric constraints.

\subsection{Discretization of the TSGV-based model}
We discretize the image domain $\Omega$ on an $N\times N$ uniform grid. A digital image, represented by a matrix $u$ with entries $u_{i_1,i_2}$ for $1 \leq i_1, i_2 \leq N$, can be rearranged column-wise into a vector $\mathbf{u} = (u_1, u_2, \ldots, u_{N^2})^{\top} \in \mathbb{R}^{N^2}$.

We begin by defining two fundamental $N \times N$ circulant matrices under periodic boundary conditions:
\[
\mathbf{D}_1=
\begin{bmatrix}
-1 & 1 & & & \\
& -1 & \ddots & & \\
& & \ddots & 1 \\
1 & & & -1
\end{bmatrix},\quad \mathbf{D}_2 = -\mathbf{D}_1^{\top}.
\]
The forward and backward difference operators in the $x$- and
$y$-directions are defined by
\[
\mathbf{D}_x^+ = \mathbf{I}_N\otimes \mathbf{D}_1,\
\mathbf{D}_y^+ = \mathbf{D}_1\otimes \mathbf{I}_N,\
\mathbf{D}_x^- = \mathbf{I}_N\otimes \mathbf{D}_2,\
\mathbf{D}_y^- = \mathbf{D}_2\otimes \mathbf{I}_N,
\]
where $\mathbf{I}_N \in \mathbb{R}^{N\times N}$ is the identity matrix, and $\otimes$ denotes the Kronecker product of two matrices. Under the above notation, the discrete gradient operator $\nabla: \mathbb{R}^{N^2}\rightarrow\mathbb{R}^{N^2}\times\mathbb{R}^{N^2}$ is defined by
\[
\nabla \mathbf{u} =
\begin{bmatrix}
\mathbf{D}_x^+\mathbf{u}\\
\mathbf{D}_y^+\mathbf{u}
\end{bmatrix}.
\]
By the standard inner products in $\mathbb{R}^{N^2}$ and $\mathbb{R}^{N^2}\times \mathbb{R}^{N^2}$, the discrete divergence operator is given by
\[
\operatorname{div}\mathbf{p}=\mathbf{D}_x^-\mathbf{p}_1 + \mathbf{D}_y^-\mathbf{p}_2,
\]
where $\mathbf{p}=\begin{bmatrix}
\mathbf{p}_1\\
\mathbf{p}_2
\end{bmatrix}\in \mathbb{R}^{2N^2}$. For notational convenience, we denote the $i$th row of $\mathbf{D}_x^+$ (respectively, $\mathbf{D}_y^+$) as $\mathbf{D}_{x,i}^+$ (respectively, $\mathbf{D}_{y,i}^+$), and $\nabla_i \mathbf{u} = \begin{bmatrix}
\mathbf{D}_{x,i}^+ \mathbf{u}\\
\mathbf{D}_{y,i}^+ \mathbf{u}
\end{bmatrix}$ for $1\leq i\leq N^2$.

Using the above discretization and a smooth approximation of the nonsmooth regularization term, we obtain the discrete form of \eqref{bilinear_model},
\begin{equation} \label{discrete_bilinear_model}
\begin{aligned}
&\min_{\mathbf{u},\mathbf{n},\mathbf{q}} E(\mathbf{u},\mathbf{n},\mathbf{q}) = \sum_{i=1}^{N^2}
\left\|\left[\begin{array}{cc}
\mathbf{D}_{x,i}^+ \mathbf{n}_1 & \mathbf{D}_{y,i}^+ \mathbf{n}_1 \\
\mathbf{D}_{x,i}^+ \mathbf{n}_2 & \mathbf{D}_{y,i}^+ \mathbf{n}_2
\end{array}\right] \right\|_{\mathrm{F},\epsilon} + \frac{\lambda}{2} \|\mathbf{u}-\mathbf{f}\|^2 \\
&\ \ \mbox{s.t.} \ \mathbf{D}_{x}^+ \mathbf{u} = \mathbf{q}\odot \mathbf{n}_1,\ \mathbf{D}_{y}^+ \mathbf{u} = \mathbf{q}\odot \mathbf{n}_2,\ \bm{\beta} = \mathbf{q}\odot \mathbf{n}_3, \\
& \ \qquad |\mathbf{n}|_{\psi_j}=\mathbf{1},\ \mathbf{q} \geq \bm{\beta},\ \mathbf{n}_3\geq0, \\
\end{aligned}
\end{equation}
where $j=1$ or $2$, with
\begin{equation*}
\begin{aligned}
&|\mathbf{n}|_{\psi_1}=\sqrt{\mathbf{n}_1\odot \mathbf{n}_1 + \mathbf{n}_2\odot \mathbf{n}_2} + \mathbf{n}_3,\\
&|\mathbf{n}|_{\psi_2}=\sqrt{\mathbf{n}_1\odot \mathbf{n}_1 + \mathbf{n}_2\odot \mathbf{n}_2 + \mathbf{n}_3\odot \mathbf{n}_3},
\end{aligned}
\end{equation*}
$\mathbf{1}$ and $\bm{\beta}$ represent two vectors whose elements are all ones and $\beta$, respectively, $\odot$ denotes the componentwise product, $\mathbf{q}=(q_1,q_2,\ldots,q_{N^2})^{\top} \geq \bm{\beta}$ means $q_i\geq \beta$ for $1\leq i\leq N^2$, $\mathbf{n}=[\mathbf{n}_1,\mathbf{n}_2,\mathbf{n}_3]^\top$, $\mathbf{n}_l=(n_{l,1},n_{l,2},\ldots,n_{l,N^2})^{\top}$, $l=1,2,3$, $\|\cdot\|$ denotes the standard $\ell^2$ norm, and $\|\cdot\|_{\mathrm F,\epsilon}$ denotes the relaxed Frobenius norm, i.e, for any $2\times2$ matrix
$H=
\begin{bmatrix}
H_{11} & H_{12}\\
H_{21} & H_{22}
\end{bmatrix}$, $\|H\|_{\mathrm F,\epsilon}=\sqrt{H_{11}^2+H_{12}^2+H_{21}^2+H_{22}^2+\epsilon}$ and $\epsilon>0$ is a small positive constant.
Then, we approximate the constrained
minimization problem \eqref{discrete_bilinear_model} by the following unconstrained optimization problem
\begin{align}
\min_{\mathbf{u},\mathbf{n},\mathbf{q}} E_{\alpha,\mathbb{I}}(\mathbf{u},\mathbf{n},\mathbf{q}) =
& \sum_{i=1}^{N^2}
\left\|\left[\begin{array}{cc}
\mathbf{D}_{x,i}^+ \mathbf{n}_1 & \mathbf{D}_{y,i}^+ \mathbf{n}_1 \\
\mathbf{D}_{x,i}^+ \mathbf{n}_2 & \mathbf{D}_{y,i}^+ \mathbf{n}_2
\end{array}\right] \right\|_{\mathrm{F},\epsilon} + \frac{\lambda}{2} \|\mathbf{u}-\mathbf{f}\|^2 \nonumber \\
& +\frac{\alpha}{2} \left(\|\mathbf{D}_{x}^+ \mathbf{u} - \mathbf{q}\odot \mathbf{n}_1\|^2 + \|\mathbf{D}_{y}^+ \mathbf{u} - \mathbf{q}\odot \mathbf{n}_2\|^2 + \|\bm{\beta} - \mathbf{q}\odot \mathbf{n}_3\|^2 \right) \label{unconstrained_discrete_bilinear_model} \\
& + \mathbb{I}_{\mathcal{S}_{j,+}}(\mathbf{n}) + \mathbb{I}_{\mathcal{R}_\beta}(\mathbf{q}), \nonumber
\end{align}
where
\begin{align*}
& \mathcal{S}_{j,+}=\left\{ \mathbf{n}=\begin{bmatrix}
\mathbf{n}_1\\
\mathbf{n}_2\\
\mathbf{n}_3
\end{bmatrix} \Bigg| \ \mathbf{n}_1,\mathbf{n}_2,\mathbf{n}_3 \in \mathbb{R}^{N^2}, \ \mathbf{n}_3\geq0,\ |\mathbf{n}|_{\psi_j}=\mathbf{1} \right\}; \\
& \mathcal{R}_\beta=\left\{ \mathbf{q} \in \mathbb{R}^{N^2} | \ \mathbf{q}\geq \bm{\beta} \right\},
\end{align*}
$\alpha$ is a positive parameter, $j=1$ or $2$, and the indicator function $\mathbb{I}$ is defined by
\begin{align*}
\mathbb{I}_{\mathcal{A}}(s)=
\left\{\begin{array}{lc}
0,       & \text{if } s\in \mathcal{A}; \\
+\infty, & \text{otherwise}.
\end{array}\right.
\end{align*}

We define a smooth function $E_{\alpha}(\mathbf{u},\mathbf{n} ,\mathbf{q})$ by
\begin{align}
E_{\alpha}(\mathbf{u},\mathbf{n},\mathbf{q}) =
& \sum_{i=1}^{N^2}
\left\|\left[\begin{array}{cc}
\mathbf{D}_{x,i}^+ \mathbf{n}_1 & \mathbf{D}_{y,i}^+ \mathbf{n}_1 \\
\mathbf{D}_{x,i}^+ \mathbf{n}_2 & \mathbf{D}_{y,i}^+ \mathbf{n}_2
\end{array}\right] \right\|_{\mathrm{F},\epsilon} + \frac{\lambda}{2} \|\mathbf{u}-\mathbf{f}\|^2 \nonumber \\
& \ + \frac{\alpha}{2} \left(\|\mathbf{D}_{x}^+ \mathbf{u} - \mathbf{q}\odot \mathbf{n}_1\|^2 + \|\mathbf{D}_{y}^+ \mathbf{u} - \mathbf{q}\odot \mathbf{n}_2\|^2 + \|\bm{\beta} - \mathbf{q}\odot \mathbf{n}_3\|^2 \right). \label{E_alpha}
\end{align}
We then write the objective function in \eqref{unconstrained_discrete_bilinear_model} as
\begin{equation*}
E_{\alpha,\mathbb{I}}(\mathbf{u},\mathbf{n},\mathbf{q}) = E_{\alpha}(\mathbf{u},\mathbf{n},\mathbf{q}) + \mathbb{I}_{\mathcal{S}_{j,+}}(\mathbf{n}) + \mathbb{I}_{\mathcal{R}_\beta}(\mathbf{q}).
\end{equation*}

For the minimization problem \eqref{unconstrained_discrete_bilinear_model}, by denoting $k=0,1,\ldots$ as the iteration number, we compute $(\mathbf{u}^{k+1},\mathbf{n}^{k+1},\mathbf{q}^{k+1})$ alternately, i.e.,
\begin{equation}\label{alm}
\left\{
\begin{aligned}
& \mathbf{u}^{k+1} = \arg\min_{\mathbf{u}} E_{\alpha,\mathbb{I}}(\mathbf{u}, \mathbf{n}^k, \mathbf{q}^k), \\
& \mathbf{n}^{k+1} = \arg\min_{\mathbf{n}} E_{\alpha,\mathbb{I}}(\mathbf{u}^{k+1}, \mathbf{n}, \mathbf{q}^k), \\
& \mathbf{q}^{k+1} = \arg\min_{\mathbf{q}} E_{\alpha,\mathbb{I}}(\mathbf{u}^{k+1}, \mathbf{n}^{k+1}, \mathbf{q}).
\end{aligned}
\right.
\end{equation}
Since $E_{\alpha,\mathbb{I}}$ is strongly convex with respect to $\mathbf{u}$ and $\mathbf{q}$, respectively, the corresponding subproblems can be solved efficiently in closed form. The $\mathbf{n}$-subproblem, however, is more challenging due to the non-convex constraint $\mathbf{n} \in \mathcal{S}_{j,+}$. In \cite{liu2024fast}, the authors solved a similar subproblem via a forward-backward splitting scheme, which provides convergence guarantees but requires the step size to satisfy $\tau \leq 1/L$, where the Lipschitz constant $L$ depends on $\mathbf{q}^k$ and may become large, forcing the step size to be prohibitively small. If the step size is not chosen appropriately, the energy-decreasing property cannot be guaranteed. To overcome this limitation, we solve the $n$-subproblem using an MM approach, which will be detailed in the next subsection.

\subsection{MM method for the $\mathbf{n}$-subproblem}
We first characterize the structure of the $\mathbf{n}$-subproblem. Fixing $\mathbf u^{k+1}$ and $\mathbf{q}^k$, denote
\[
\phi(\mathbf{n})=\sum_{i=1}^{N^2}
\left\|
\begin{bmatrix}
\mathbf D_{x,i}^+ \mathbf n_1 & \mathbf D_{y,i}^+ \mathbf n_1 \\
\mathbf D_{x,i}^+ \mathbf n_2 & \mathbf D_{y,i}^+ \mathbf n_2
\end{bmatrix}
\right\|_{\mathrm F,\epsilon}.
\]
The $\mathbf n$-subproblem associated with the energy $E_{\alpha,\mathbb{I}}$ is then given by
\begin{multline} \label{n-subproblem}
    E_{\alpha,\mathbb{I}}(\mathbf u^{k+1}, \mathbf n, \mathbf q^k) = \phi(\mathbf{n}) + \mathbb{I}_{\mathcal S_{j,+}}(\mathbf n) \\
    + \frac{\alpha}{2}
\Big(
\|\mathbf D_x^+ \mathbf u^{k+1} - \mathbf q^k \odot \mathbf n_1\|^2
+ \|\mathbf D_y^+ \mathbf u^{k+1} - \mathbf q^k \odot \mathbf n_2\|^2
+ \|\boldsymbol\beta - \mathbf q^k \odot \mathbf n_3\|^2
\Big).
\end{multline}
The quadratic term in \eqref{n-subproblem} can be rewritten as
\[
h^k(\mathbf n) = \frac{\alpha}{2} \left\|\mathbf q^k\odot\mathbf n - \begin{bmatrix}
\mathbf D_{x}^+ \mathbf u^{k+1} \\
\mathbf D_{y}^+ \mathbf u^{k+1} \\
\bm{\beta}
\end{bmatrix} \right\|^2.
\]
Since $h^k$ is strictly convex and quadratic in $\mathbf{n}$, it admits the exact second-order expansion
\begin{equation}\label{second_order_h}
h^k(\mathbf n)
= h^k(\mathbf n^k) + \left< \nabla_\mathbf{n}h^k(\mathbf{n}^k), \mathbf{n}-\mathbf{n}^k\right> + \frac{1}{2} (\mathbf{n}-\mathbf{n}^k)^\top H^k(\mathbf{n}-\mathbf{n}^k),
\end{equation}
where
\[
\nabla_\mathbf{n}h^k(\mathbf{n}^k) = \alpha (\mathbf q^k)^2
\left(\mathbf n^k - \frac{1}{\mathbf q^k}
\begin{bmatrix}
\mathbf D_{x}^+ \mathbf u^{k+1} \\
\mathbf D_{y}^+ \mathbf u^{k+1} \\
\bm{\beta}
\end{bmatrix}\right),
\]
and $H^k=\alpha\mathrm{diag}((\mathbf q^k)^2, (\mathbf q^k)^2, (\mathbf q^k)^2)$ with $H_i^k=\alpha\mathrm{diag}((\mathbf q_i^k)^2, (\mathbf q_i^k)^2, (\mathbf q_i^k)^2)$. To handle the term $\phi(\mathbf{n})$, we note that its gradient is Lipschitz continuous with constant $L_\phi$. This Lipschitz condition directly implies the following quadratic upper bound:
\begin{equation}\label{second_order_phi}
\phi(\mathbf{n}) \leq \phi(\mathbf{n}^k) + \left< \nabla_{\mathbf{n}}\phi(\mathbf{n}^k), \mathbf{n}-\mathbf{n}^k\right> + \frac{\mu}{2}\|\mathbf{n}-\mathbf{n}^k\|^2,
\end{equation}
where $\mu>L_\phi$. The strict choice $\mu>L_\phi$ will be used later to derive a sufficient descent property for the $\mathbf{n}$-subproblem.

Combining the exact second-order expansion of $h^k$ in \eqref{second_order_h} with the majorization of $\phi$ in \eqref{second_order_phi}, we construct the following function for the $\mathbf{n}$-subproblem,
\begin{align}
Q(\mathbf n \mid \mathbf n^k)
=& h^k(\mathbf n^k) + \left< \nabla_\mathbf{n}h^k(\mathbf{n}^k), \mathbf{n}-\mathbf{n}^k\right> + \frac{1}{2} (\mathbf{n}-\mathbf{n}^k)^\top H^k(\mathbf{n}-\mathbf{n}^k) \nonumber \\
& + \phi(\mathbf{n}^k) + \left< \nabla_{\mathbf{n}}\phi(\mathbf{n}^k), \mathbf{n}-\mathbf{n}^k\right> + \frac{\mu}{2}\|\mathbf{n}-\mathbf{n}^k\|^2 + \mathbb{I}_{\mathcal S_{j,+}}(\mathbf n) \label{Qn}.
\end{align}
The following lemma proves that $Q(\cdot \mid \mathbf{n}^k)$ is an MM surrogate function of $E_{\alpha,\mathbb{I}}$ with respect to $\mathbf{n}$.

\begin{lemma} \label{lem:mm-surrogate}
The function $Q(\cdot \mid \mathbf n^k)$ defined in \eqref{Qn} is an MM surrogate function of $E_{\alpha,\mathbb{I}}$ in the sense that
\begin{equation}\label{n_decrease}
\begin{aligned}
Q(\mathbf n \mid \mathbf n^k) &\geq E_{\alpha,\mathbb{I}}(\mathbf u^{k+1}, \mathbf n, \mathbf q^k)+ \frac{\mu-L_\phi}{2}\|\mathbf{n}-\mathbf{n}^k\|^2,\\
Q(\mathbf n^k \mid \mathbf n^k) &= E_{\alpha,\mathbb{I}}(\mathbf u^{k+1}, \mathbf n^k, \mathbf q^k).
\end{aligned}
\end{equation}
\end{lemma}
\begin{proof}
The result follows directly from substituting \eqref{second_order_h} and \eqref{second_order_phi} into \eqref{n-subproblem}.
\end{proof}

As an immediate consequence of Lemma~\ref{lem:mm-surrogate}, the MM update for the $\mathbf{n}$-subproblem inherits a natural energy descent property, as stated in the following proposition.


\begin{proposition}\label{prop:energy-descent-n}
Let
\[
\mathbf n^{k+1} \in \arg\min_{\mathbf n} Q(\mathbf n \mid \mathbf n^k).
\]
Then the energy associated with the $\mathbf n$-subproblem satisfies
\[
E_{\alpha,\mathbb{I}}(\mathbf u^{k+1}, \mathbf n^k, \mathbf q^k)
\geq E_{\alpha,\mathbb{I}}(\mathbf u^{k+1}, \mathbf n^{k+1}, \mathbf q^k) + \frac{\mu-L_\phi}{2} \|\mathbf n^{k+1}-\mathbf n^k\|^2.
\]
\end{proposition}
\begin{proof}
By Lemma~\ref{lem:mm-surrogate} and the optimality of $\mathbf{n}^{k+1}$:
\begin{align*}
&E_{\alpha,\mathbb{I}}(\mathbf u^{k+1}, \mathbf n^{k+1}, \mathbf q^k) + \frac{\mu-L_\phi}{2} \|\mathbf n^{k+1}-\mathbf n^k\|^2 \\
\leq &  Q(\mathbf n^{k+1} \mid \mathbf n^k) \leq Q(\mathbf n^k \mid \mathbf n^k) =
E_{\alpha,\mathbb{I}}(\mathbf u^{k+1}, \mathbf n^k, \mathbf q^k).
\end{align*}
\end{proof}

This descent property follows from the MM construction and does not rely on the particular algorithm adopted. Consequently, any algorithm that exactly or sufficiently decreases $Q(\cdot \mid \mathbf n^k)$ can be employed to solve the $\mathbf{n}$-subproblem, which provides considerable flexibility in the choice of inner solver.

\subsection{Implementation of MMAMM}
Embedding the MM method into the alternating minimization framework \eqref{alm}, we update $(\mathbf{u}^{k+1},\mathbf{n}^{k+1},\mathbf{q}^{k+1})$, for $k=0,1,\ldots$, in the following framework
\begin{equation*}
\left\{
\begin{aligned}
& \mathbf{u}^{k+1} = \arg\min_{\mathbf{u}} E_{\alpha,\mathbb{I}}(\mathbf{u}, \mathbf{n}^k, \mathbf{q}^k), \\
& \mathbf{n}^{k+1} = \arg\min_{\mathbf{n}} Q(\mathbf n \mid \mathbf n^k), \\
& \mathbf{q}^{k+1} = \arg\min_{\mathbf{q}} E_{\alpha,\mathbb{I}}(\mathbf{u}^{k+1}, \mathbf{n}^{k+1}, \mathbf{q}).
\end{aligned}
\right.
\end{equation*}
The $\mathbf{u}$-subproblem is
\begin{equation*}
\mathbf{u}^{k+1} = \arg\min_{\mathbf{u}} \frac{\lambda}{2}\|\mathbf{u} - \mathbf{f}\|^2 + \frac{\alpha}{2} \big(\|\mathbf{D}_{x}^+ \mathbf{u} - \mathbf{q}^k\odot \mathbf{n}_1^k\|^2  + \|\mathbf{D}_{y}^+ \mathbf{u} - \mathbf{q}^k\odot \mathbf{n}_2^k\|^2 \big).
\end{equation*}
The unique minimum solution is given by
\begin{equation} \label{u_solution}
\begin{aligned}
\mathbf{u}^{k+1} = & \left(\lambda\mathbf{I}_{N^2} - \alpha\left(\mathbf{D}_x^- \mathbf{D}_x^+ + \mathbf{D}_y^- \mathbf{D}_y^+\right)\right)^{-1} \\
& \ \times \left(\lambda\mathbf{f} - \alpha\left(\mathbf{D}_x^-(\mathbf{q}^k \odot \mathbf{n}_1^k) + \mathbf{D}_y^-(\mathbf{q}^k \odot \mathbf{n}_2^k)\right)\right).
\end{aligned}
\end{equation}
Under periodic boundary conditions, the linear system can be solved efficiently by FFT.

For the $\mathbf{n}$-subproblem, minimizing $Q(\mathbf{n}\mid\mathbf{n}^k)$ yields
\begin{equation*}
\begin{aligned}
\mathbf{n}^{k+1}
&= \arg\min_{\mathbf{n}} Q(\mathbf n \mid \mathbf n^k) \\
&= \arg\min_{\mathbf{n}\in \mathcal S_{j,+}} \left< \nabla_\mathbf{n}(h^k+\phi)(\mathbf{n}^k), \mathbf{n}-\mathbf{n}^k\right> + \frac{1}{2} (\mathbf{n}-\mathbf{n}^k)^\top \left(\mu\mathbf{I}+H^k\right)(\mathbf{n}-\mathbf{n}^k) \nonumber \\
&= \arg\min_{\mathbf{n}\in \mathcal S_{j,+}} \frac{1}{2}\left\|\mathbf{n} - \mathbf{n}^k + (\mu\mathbf{I}+H^k)^{-1} \nabla_\mathbf{n}(h^k+\phi)(\mathbf{n}^k) \right\|^2_{\mu\mathbf{I}+H^k},
\end{aligned}
\end{equation*}
where $\|\mathbf{v}\|^2_\Lambda := \mathbf{v}^\top \Lambda \mathbf{v}$, $\Lambda^k := \mu\mathbf{I}+H^k$ and $\mathbf{n}^{k+\frac{1}{2}} := \mathbf{n}^k - (\Lambda^k)^{-1} \nabla_\mathbf{n}(h^k+\phi)(\mathbf{n}^k)$. Since $\Lambda^k$ is a block-diagonal matrix, the weighted norm $\|\cdot\|_{\Lambda^k}^2$ decouples across pixels, reducing the minimization to $N^2$ independent subproblems of the form
\begin{equation*}
\begin{aligned}
[\mathbf{n}^{k+1}]_i
&= \arg\min_{\mathbf{n}\in \mathcal S^1_{j,+}} \frac{1}{2}\left\|\mathbf{n} - [\mathbf{n}^{k+\frac{1}{2}}]_i \right\|^2\\
&= \operatorname{Proj}_{\mathcal{S}_{j,+}^1} \left([\mathbf{n}^{k+\frac{1}{2}}]_i\right),
\end{aligned}
\end{equation*}
where $1\leq i\leq N^2$, $j=1,2$,
\begin{equation*}
\begin{aligned}
\mathcal{S}_{1,+}^1 &= \{(x,y,z)^\top \in \mathbb{R}^3 | \sqrt{x^2 + y^2} + z = 1,\ z\geq0\},\\
\mathcal{S}_{2,+}^1 &= \{(x,y,z)^\top \in \mathbb{R}^3 | x^2 + y^2 + z^2 = 1,\ z\geq0\},
\end{aligned}
\end{equation*}
with $[\mathbf{n}^{k+1}]_i=[n_{1,i}^{k+1},n_{2,i}^{k+1},n_{3,i}^{k+1}]^\top$ and $\operatorname{Proj}_{\mathcal{S}_{j,+}^1}$ represents the projection operator onto an upper--cone surface ${\mathcal{S}_{1,+}^1}$ or an upper--hemispherical surface ${\mathcal{S}_{2,+}^1}$.

For $j=1$, we begin with the projection onto the whole cone surface
\[
\mathcal{S}_1^1 = \{(x,y,z)^\top \in \mathbb{R}^3 | \sqrt{x^2 + y^2} + z = 1\},
\]
and then incorporate the additional upper-surface constraint. The projection point of $P_0=[x_0,y_0,z_0]^\top$ onto the surface $\mathcal{S}_1^1$ is denoted by $P_p=[x_p,y_p,z_p]^\top$. The outward normal vector to the cone $\mathcal{S}_1^1$ at $P_p$ is given by $[\frac{x_p}{r_p},\frac{y_p}{r_p},1]^\top$, where $r_p=\sqrt{x_p^2+y_p^2}$ and $r_p\neq 0$ (the case $r_p=0$ is a special case, we will discuss separately). The normal line passing through $P_p$ and $P_0$ satisfies the following equation
\[
\frac{x_0-x_p}{\frac{x_p}{r_p}} = \frac{y_0-y_p}{\frac{y_p}{r_p}} = \frac{z_0-z_p}{1}.
\]
This equation can be rewritten as
\begin{equation} \label{normal_line}
\left\{
\begin{aligned}
& x_0-x_p = a\frac{x_p}{r_p},\\
& y_0-y_p = a\frac{y_p}{r_p},\\
& z_0-z_p = a,\\
& r_p+z_p = 1,
\end{aligned}
\right.
\end{equation}
where $a$ is a scalar constant. From the third and fourth equations in \eqref{normal_line}, we have
\[
a=z_0-z_p=z_0+r_p-1.
\]
Let $r_0=\sqrt{x_0^2+y_0^2}$. Using the first and second equations of \eqref{normal_line}, we derive the relationship between $r_0$ and $r_p$
\begin{equation} \label{radial}
r_0 = r_p \left| 1+\frac{a}{r_p} \right| = r_p \left| 1+\frac{z_0+r_p-1}{r_p} \right| = \left| 2r_p+z_0-1 \right|.
\end{equation}
We now proceed to analyze the absolute value by considering the following three cases:

\begin{description}
\item[Case1:] $r_p\neq 0$ and $2r_p+z_0-1>0$.

Under this condition, the equation \eqref{radial} yields
\[
r_p=\frac{r_0-z_0+1}{2}.
\]
Since $r_p\geq0$, it follows that $r_0-z_0+1\geq0$. The equality $r_0 - z_0 + 1 = 0$ leads to a special case that will be discussed separately. In the present case, we assume $0<r_0-z_0+1$.

\begin{description}
\item[Case1a:] $0<r_0 - z_0 + 1\leq2$.

From the first equation in \eqref{normal_line} and the expression of $r_0$, we have
\[
x_0 = x_p \left( 1+\frac{a}{r_p} \right) = x_p \left( 2+\frac{z_0-1}{r_p} \right) = x_p\frac{r_0}{r_p},
\]
which means
\[
x_p=x_0\frac{r_p}{r_0}=x_0\frac{r_0-z_0+1}{2r_0}.
\]
From the second and the fourth equations in \eqref{normal_line}, we obtain
\[
y_p=y_0\frac{r_0-z_0+1}{2r_0}, \quad z_p=1-r_p=\frac{1+z_0-r_0}{2}.
\]
Consequently, the projection point of $P_0$ onto $\mathcal{S}_{1,+}^1$ is given by
\[
P_p = \left[x_0\frac{r_0-z_0+1}{2r_0}, y_0\frac{r_0-z_0+1}{2r_0}, \frac{1+z_0-r_0}{2}\right]^{\top}.
\]

\item[Case1b:] $r_0 - z_0 + 1>2$.

The projection point obtained above satisfies $z_p=\frac{1+z_0-r_0}{2}<0$. This point does not satisfy the constraint $z_p\geq0$. Hence, the projection point must lie on the boundary circle and has the same azimuthal direction as $(x_0,y_0)$. Therefore, the projection of $P_0$ onto $\mathcal{S}_{1,+}^1$ is
\[
P_p=
\left[
\frac{x_0}{r_0},
\frac{y_0}{r_0},
0
\right]^{\top}.
\]
\end{description}

\item[Case2:] $r_p\neq 0$ and $2r_p+z_0-1<0$.

Equation \eqref{radial} gives
\[
r_p=\frac{1-r_0-z_0}{2},
\]
with $1-r_0-z_0>0$ to ensure $r_p>0$. This condition means that $P_0$ lies the interior region associated with the whole cone surface. For a valid projection onto the cone, $P_0$ and $P_p$ should share the same azimuthal angle in the $xy$-plane, implying $x_0$ and $x_p$ (and similarly $y_0$ and $y_p$) must have the same sign. From the first equation in \eqref{normal_line},
\[
x_0 = x_p \left( 1+\frac{a}{r_p} \right) = x_p \left( 2+\frac{z_0-1}{r_p} \right) = -x_p\frac{r_0}{r_p},
\]
since $\frac{r_0}{r_p}>0$, it follows that $x_0$ and $x_p$ have opposite signs, contradicting the fact that $x_0$ has the same sign as $x_p$. Therefore, \textbf{Case 2} is invalid for this projection problem.

\item[Case3:] $r_p\neq 0$ and $2r_p+z_0-1=0$.

In this case, we obtain $r_p=\frac{1-z_0}{2}$, and assume $z_0<1$ (the case $z_0\geq1$ will be discussed separately). 

\begin{description}
\item[Case3a:] $-1\leq z_0<1$.

The condition $r_0=0$ implies $P_0=[0,0,z_0]^\top$. From the fourth equation in \eqref{normal_line},
\[
z_p=1-r_p=\frac{1+z_0}{2}.
\]
Thus, the projection point is
\[
P_p=[x_p,y_p,\frac{1+z_0}{2}]^\top, \quad x_p^2+y_p^2=\left(\frac{1-z_0}{2}\right)^2.
\]
\item[Case3b:] $z_0<-1$.

The projection point obtained above satisfies $z_p=\frac{1+z_0}{2}<0$. Hence, the projection point must lie on the boundary circle and
\[
P_p=[x_p,y_p,0]^{\top}, \quad x_p^2+y_p^2=1.
\]
\end{description}
\end{description}
We now address the remaining special subcases arising from the conditions in \textbf{Case1} and \textbf{Case3}: \textbf{Case1} with $r_0-z_0+1=0$; \textbf{Case3} with $1-z_0\leq0$. Both situations are encompassed by the case $r_p=0$.

\begin{description}
\item[Case4:] $r_p=0$.

If $r_0-z_0+1=0$ in \textbf{Case1}, then $r_p=0$. If $1-z_0\leq0$ in \textbf{Case3}, the relation $r_p=\frac{1-z_0}{2}$ forces $r_p=0$ and consequently $z_0=1$.
Geometrically, these conditions imply that $P_0$ lies inside or on the cone
\[
\bar{\mathcal{S}}_1^1 = \{(x,y,z)^\top \in \mathbb{R}^3 : z-\sqrt{x^2 + y^2} = 1\},
\]
which is symmetric to the whole cone surface $\mathcal S_1^1$ with respect to the plane $z=1$. For such $P_0$, the projection onto $\mathcal{S}_{1,+}^1$ is the apex, i.e., $P_p=[0,0,1]^\top$.
\end{description}

Based on the above discussion, the solution of $\mathbf{n}$-subproblem for $j=1$ is
\begin{equation}\label{n_solution_1}
[\mathbf{n}^{k+1}]_i =
\begin{cases}
[x_0\frac{r_0-z_0+1}{2r_0},y_0\frac{r_0-z_0+1}{2r_0},\frac{1+z_0-r_0}{2}] & \text{if } r_0>0,\ 0<r_0-z_0+1\leq2, \\
[0,0,1] & \text{if } r_0>0,\ r_0-z_0+1\leq0, \\
[\frac{x_0}{r_0},\frac{y_0}{r_0},0] & \text{if } r_0>0,\ 2<r_0-z_0+1, \\
[\frac{1-z_0}{2},0,\frac{1+z_0}{2}] & \text{if } r_0=0, -1\leq z_0<1,\\
[0,0,1] & \text{if } r_0=0,\ z_0\geq1, \\
[\mathbf{d}_1^\top,0] & \text{if } r_0=0,\ z_0<-1,
\end{cases}
\end{equation}
and for $j=2$, the closed-form solutions are
\begin{equation}\label{n_solution_2}
[\mathbf{n}^{k+1}]_i =
\begin{cases}
\frac{[x_0,y_0,z_0]}{\sqrt{x_0^2+y_0^2+z_0^2}} & \text{if } z_0\geq0,\ x_0^2+y_0^2+z_0^2 \neq 0, \\
\mathbf{d}_2 & \text{if } z_0=0,\ r_0=0, \\
[\frac{x_0}{r_0},\frac{y_0}{r_0},0] & \text{if } z_0<0,\ r_0>0, \\
[1,0,0] & \text{if } z_0<0,\ r_0=0, \\
\end{cases}
\end{equation}
where $r_0=\sqrt{x_0^2+y_0^2}$, $[x_0,y_0,z_0] = [n_{1,i}^{k+\frac{1}{2}},n_{2,i}^{k+\frac{1}{2}},n_{3,i}^{k+\frac{1}{2}}]$, $\mathbf d_1\in\mathbb R^2$ and $\mathbf{d}_2\in\mathcal{S}_{2,+}^1$ are arbitrary unit vectors.

The $\mathbf{q}$-subproblem is
\begin{equation*}
\begin{aligned}
\mathbf{q}^{k+1} = \arg\min_{\mathbf{q}\geq\bm{\beta}} & \frac{\alpha}{2} \big(\|\mathbf{D}_{x}^+ \mathbf{u}^{k+1} - \mathbf{q}\odot \mathbf{n}_1^{k+1}\|^2 + \|\mathbf{D}_{y}^+ \mathbf{u}^{k+1} - \mathbf{q}\odot \mathbf{n}_2^{k+1}\|^2 \\
&\quad+ \|\bm{\beta} - \mathbf{q}\odot \mathbf{n}_3^{k+1}\|^2 \big),
\end{aligned}
\end{equation*}
which is also separable. For $1\leq i\leq N^2$, it can be decomposed into the following $N^2$ independent minimization problems
\begin{equation*}
\begin{aligned}
q_i^{k+1} = \arg\min_{q_i\geq\beta} & \frac{\alpha}{2} \left((\mathbf{D}_{x,i}^+ \mathbf{u}^{k+1} - q_i n_{1,i}^{k+1})^2 + (\mathbf{D}_{y,i}^+ \mathbf{u}^{k+1} - q_i n_{2,i}^{k+1})^2 + (\beta - q_i n_{3,i}^{k+1})^2 \right) \\
= \arg\min_{q_i\geq\beta} & \frac{\alpha}{2}q_i^2 ((n_{1,i}^{k+1})^2+(n_{2,i}^{k+1})^2+(n_{3,i}^{k+1})^2) \\
& -\alpha q_i(\mathbf{D}_{x,i}^+ \mathbf{u}^{k+1} n_{1,i}^{k+1} + \mathbf{D}_{y,i}^+ \mathbf{u}^{k+1} n_{2,i}^{k+1} + \beta n_{3,i}^{k+1}),
\end{aligned}
\end{equation*}
where $(n_{1,i}^{k+1})^2+(n_{2,i}^{k+1})^2+(n_{3,i}^{k+1})^2\neq0$, by using \eqref{n_solution_1} (or \eqref{n_solution_2}). It has the following closed-form solution
\begin{equation}\label{q_solution}
q_i^{k+1} = \max\left(\beta,\ \frac{\mathbf{D}_{x,i}^+ \mathbf{u}^{k+1} n_{1,i}^{k+1} + \mathbf{D}_{y,i}^+ \mathbf{u}^{k+1} n_{2,i}^{k+1} + \beta n_{3,i}^{k+1}}{(n_{1,i}^{k+1})^2+(n_{2,i}^{k+1})^2+(n_{3,i}^{k+1})^2}\right).
\end{equation}
We summarize the proposed algorithm for solving \eqref{unconstrained_discrete_bilinear_model} in Algorithm \ref{algorithm}.

\begin{algorithm}[H]
\caption{MMAMM algorithm for solving TSGV-based bilinear decomposition model \eqref{unconstrained_discrete_bilinear_model}} \label{algorithm}
\SetAlgoLined
\KwIn{The observation image $\mathbf{f}$, the model parameter $\lambda$, the scaling function $\psi$ with parameter $\beta$ and algorithm parameters $\epsilon$, $\mu$, $\alpha$.}
\KwOut{$\mathbf{u}^k$.}
\textbf{Initialization:} $\mathbf{u}^0 = \mathbf{f}$, $q^0_i = \frac{1}{\psi(|\nabla_i \mathbf{u}^0|)}$, \\ $\mathbf{n}_i^0 = \operatorname{Proj}_{\mathcal{S}_{j,+}^1}([\psi(|\nabla_i \mathbf{u}^0|) \mathbf{D}_{x,i}^+ \mathbf{u}^0, \psi(|\nabla_i \mathbf{u}^0|) \mathbf{D}_{y,i}^+ \mathbf{u}^{0}, \psi(|\nabla_i \mathbf{u}^0|) \beta])$, \\$i=1,2,\ldots,N^{2}$, $j=1,2$.

Set $k = 0$.

\While{the termination condition is not satisfied}
{
\;\;Compute $\mathbf{u}^{k+1}$ by \eqref{u_solution} using FFT;\\
Compute $\mathbf{n}^{k+1}$ by \eqref{n_solution_1} (or \eqref{n_solution_2});\\
Compute $\mathbf{q}^{k+1}$ by \eqref{q_solution};\\
$k \gets k + 1$.
}
\Return{$\mathbf{u}^k$}.
\end{algorithm}


%% file: sections/sec-convergence-analysis.tex
\section{Convergence analysis of MMAMM} \label{sec:convergence}
We first establish Lipschitz continuity of the partial derivatives of the energy functional $E_{\alpha}(\mathbf{u},\mathbf{n},\mathbf{q})$ with respect to each variable. These constants are crucial for ensuring the convergence of the alternating minimization algorithm.
\begin{lemma}
For $E_{\alpha}(\mathbf{u},\mathbf{n},\mathbf{q})$ defined in \eqref{E_alpha}, the partial derivatives with respect to each variable are Lipschitz continuous. Explicit upper bounds for the Lipschitz constants, denoted by $L_u,\ L_q$, and $L_n$ are
\begin{equation*}
\begin{aligned}
L_u &= \lambda_{\max}\left(\lambda\mathbf{I}_{N^2} - \alpha\left(\mathbf{D}_x^- \mathbf{D}_x^+ + \mathbf{D}_y^- \mathbf{D}_y^+\right)\right),\\
L_q &= \alpha\|\mathbf{n}_1\odot\mathbf{n}_1 + \mathbf{n}_2\odot\mathbf{n}_2 + \mathbf{n}_3\odot\mathbf{n}_3\|_{\infty},\\
L_n &= \frac{8}{\sqrt{\epsilon}} + \alpha \|\mathbf{q}\|_{\infty}^{2},
\end{aligned}
\end{equation*}
where $\lambda_{\max}(\cdot)$ denotes the largest eigenvalue of a matrix and $\|\cdot\|_{\infty}$ means the $\ell^{\infty}$ norm of a vector in the Euclidean space.
\end{lemma}

\begin{proof}
The partial derivatives of $E_{\alpha}$ with respect to $\mathbf{u}$ and $\mathbf{q}$ are computed as
\begin{equation*}
\begin{aligned}
\nabla_\mathbf{u}E_{\alpha}(\mathbf{u},\mathbf{n},\mathbf{q})
&=\lambda(\mathbf{u}-\mathbf{f})-\alpha (\mathbf{D}_x^-(\mathbf{D}_x^+\mathbf{u}-\mathbf{q}\odot\mathbf{n_1}) + \mathbf{D}_y^-(\mathbf{D}_y^+\mathbf{u}-\mathbf{q}\odot\mathbf{n_2})),\\
\nabla_\mathbf{q}E_{\alpha}(\mathbf{u},\mathbf{n},\mathbf{q})
&=\alpha (\mathbf{n_1}\odot(\mathbf{q}\odot\mathbf{n_1}-\mathbf{D}_x^+\mathbf{u}) + \mathbf{n_2}\odot(\mathbf{q}\odot\mathbf{n_2}-\mathbf{D}_y^+\mathbf{u}) \\
&\quad + \mathbf{n_3}\odot(\mathbf{q}\odot\mathbf{n_3}-\bm{\beta})),\\
\end{aligned}
\end{equation*}
The Lipschitz constants of $\nabla_\mathbf{u}E_{\alpha}$ and $\nabla_\mathbf{q}E_{\alpha}$ are
\begin{equation*}
\begin{aligned}
L_u &= \lambda_{\max}\left(\lambda\mathbf{I} \otimes \mathbf{I} - \alpha\left(\mathbf{D}_x^- \mathbf{D}_x^+ + \mathbf{D}_y^- \mathbf{D}_y^+\right)\right),\\
L_q(\mathbf{n}) &= \alpha\|\mathbf{n}_1\odot\mathbf{n}_1 + \mathbf{n}_2\odot\mathbf{n}_2 + \mathbf{n}_3\odot\mathbf{n}_3\|_{\infty}.
\end{aligned}
\end{equation*}
We now estimate the Lipschitz constant for $\nabla_\mathbf{n}E_{\alpha}$. The energy $E_{\alpha}$ with respect to $\mathbf{n}$ decomposes into two parts: the regularization term and the quadratic penalty term. For each pixel $1\leq i\leq N^2$, define
\[
\mathbf{n}_T = [\mathbf{n}_1, \mathbf{n}_2]^\top, \quad
\mathbf{K}_i =
\begin{bmatrix}
\mathbf{D}_{x,i}^+ & 0\\
\mathbf{D}_{y,i}^+ & 0\\
0 & \mathbf{D}_{x,i}^+\\
0 & \mathbf{D}_{y,i}^+
\end{bmatrix},
\]
and set $g_i(\mathbf{n}) = \mathbf{K}_i \mathbf{n}_T\in \mathbb{R}^4$, where $\mathbf{n}_T$
denotes the tangential part of the normal field
$\mathbf n=[n_1,n_2,n_3]^T$. Then the regularization term reads as
\[
R(\mathbf{n}) = \sum_{i=1}^{N^2} \sqrt{\|g_i(\mathbf{n})\|^2 + \epsilon},
\]
and the quadratic penalty term can be written as
\[
T(\mathbf{n}) = \frac{\alpha}{2} \left(\|\mathbf{D}_{x}^+ \mathbf{u} - \mathbf{q}\odot \mathbf{n}_1\|^2 + \|\mathbf{D}_{y}^+ \mathbf{u} - \mathbf{q}\odot \mathbf{n}_2\|^2 + \|\bm{\beta} - \mathbf{q}\odot \mathbf{n}_3\|^2 \right).
\]
We first analyze $R(\mathbf{n})$. Consider the scalar function $\varphi(g) = \sqrt{\|g\|^2 + \epsilon}$ for $g\in \mathbb{R}^4$. Its gradient and Hessian with respect to $g$ are
\[
\nabla_g \varphi(g) = \frac{g}{\sqrt{\|g\|^2 + \epsilon}},
\]
and
\[
\nabla_g^2 \varphi(g) = \frac{1}{\sqrt{\|g\|^2 + \epsilon}} \left(
\mathbf{I} - \frac{g g^\top}{\|g\|^2 + \epsilon}
\right),
\]
respectively. The Hessian $\nabla_g^2 \varphi (g)$ is symmetric and its eigenvalues can be determined explicitly. Along the direction of $g$, let the unit vector $v=\frac{g}{\|g\|}$ (if $g=0$, take an arbitrary vector), then
\begin{equation*}
\begin{aligned}
\nabla_g^2 \varphi (g) v
&= \frac{1}{\sqrt{\|g\|^2 + \epsilon}} \left(
v - \frac{g (g^\top v)}{\|g\|^2 + \epsilon}\right) = \frac{1}{\sqrt{\|g\|^2 + \epsilon}} \left(
v - \frac{\|g\| g}{\|g\|^2 + \epsilon}\right)\\
&= \frac{1}{\sqrt{\|g\|^2 + \epsilon}} \left(
v - \frac{\|g\|^2}{\|g\|^2 + \epsilon}v\right) = \frac{\epsilon}{(\|g\|^2 + \epsilon)^{3/2}}v.
\end{aligned}
\end{equation*}
Denote $\lambda_\parallel$ as the eigenvalue of $\nabla_g^2 \varphi (g)$ in the direction $g$, thus
\[
\lambda_\parallel = \frac{\epsilon}{(\|g\|^2 + \epsilon)^{3/2}}.
\]
For any vector $v_\perp$ orthogonal to $g$ ($g^\top v_\perp=0$), we have
\[
\nabla_g^2 \varphi (g) v_\perp = \frac{1}{\sqrt{\|g\|^2 + \epsilon}} v_\perp.
\]
The corresponding eigenvalue is
\[
\lambda_\perp = \frac{1}{\sqrt{\|g\|^2 + \epsilon}}.
\]
The spectral norm of $\nabla_g^2 \varphi(g)$ satisfies the upper bound
\[
\|\nabla_g^2 \varphi(g)\| = \max\{\lambda_\parallel,\lambda_\perp\} = \lambda_\perp \leq \frac{1}{\sqrt{\epsilon}}.
\]
Since $g_i(\mathbf{n})$ is linear, the Hessian of $R(\mathbf{n})$ with respect to $\mathbf{n}$ can be expressed via the chain rule.
\[
\nabla_\mathbf{n}^2 R(\mathbf{n}) = diag\left[ \mathbf{K}_i^\top \nabla_g^2 \varphi(g_i) \mathbf{K}_i \right]_{i=1}^{N^2}.
\]
Difference operator norm estimate gives
\[
\|\mathbf{D}_x^+\| \leq 2, \quad \|\mathbf{D}_y^+\| \leq 2 \implies \|\mathbf{K}\|^2 \leq \|\mathbf{D}_x^{+}\|^2 + \|\mathbf{D}_y^{+}\|^2 = 8.
\]
Combining with the single-pixel upper bound, the spectral norm of the Hessian of $R(\mathbf{n})$ satisfies
\begin{equation}\label{R_upperbound}
\|\nabla_\mathbf{n}^2 R(\mathbf{n})\| \leq \max_i\|\nabla_g^2\varphi(g_i)\|\cdot\|\mathbf{K}\|^2 \leq \frac{8}{\sqrt{\epsilon}}.
\end{equation}
We now turn to the quadratic term $\nabla_\mathbf{n}T(\mathbf{n})$, its gradient with respect to $\mathbf{n}$ is
\[
\nabla_\mathbf{n} T(\mathbf{n}) = \alpha \begin{bmatrix}
\mathbf{q} \odot (\mathbf{q} \odot \mathbf{n}_1 - \mathbf{D}_{x}^+ \mathbf{u}) \\
\mathbf{q} \odot (\mathbf{q} \odot \mathbf{n}_2 - \mathbf{D}_{y}^+ \mathbf{u}) \\
\mathbf{q} \odot (\mathbf{q} \odot \mathbf{n}_3 - \bm{\beta})
\end{bmatrix},
\]
and the Hessian matrix is a block diagonal matrix
\[
\nabla_\mathbf{n}^2 T = \alpha \begin{bmatrix}
\mathrm{diag}(\mathbf{q}^2) & \mathbf{0} & \mathbf{0}\\
\mathbf{0} & \mathrm{diag}(\mathbf{q}^2) & \mathbf{0}\\
\mathbf{0} & \mathbf{0} & \mathrm{diag}(\mathbf{q}^2)
\end{bmatrix}.
\]
Its spectral norm  is
\begin{equation}\label{T_upperbound}
\|\nabla_\mathbf{n}^2 T\| = \alpha \|\mathbf{q}\|_\infty^2.
\end{equation}
Combining \eqref{R_upperbound} and \eqref{T_upperbound}, we obtain
\[
\|\nabla_\mathbf{n}^2 (R + T)\| \leq \frac{8}{\sqrt{\epsilon}} + \alpha \|\mathbf{q}\|_\infty^2.
\]
Therefore, the Lipschitz constant of $\nabla_{\mathbf{n}} E_\alpha = \nabla_{\mathbf{n}} (R + T)$ satisfies
\[
L_n(\mathbf{q}) = \frac{8}{\sqrt{\epsilon}} + \alpha \|\mathbf{q}\|_{\infty}^{2}.
\]
\end{proof}

With the Lipschitz constants established, we can proceed to analyze the descent properties of the MMAMM. The following lemma shows that the energy sequence is nonincreasing along the iterations.

\begin{lemma}\label{lem:energy-descent}
Let the sequence $\{(\mathbf{u}^{k+1},\mathbf{n}^{k+1},\mathbf{q}^{k+1})\}_{k=0}^{+\infty}$ be generated by the proposed MMAMM. The energy sequence $\{E_{\alpha,\mathbb{I}}(\mathbf{u}^{k+1},\mathbf{n}^{k+1},\mathbf{q}^{k+1})\}_{k=0}^{+\infty}$ satisfies a sufficient descent property.
\end{lemma}
\begin{proof}
The decreasing property of the energy with respect to the variable $\mathbf{n}$ follows directly from Proposition~\ref{prop:energy-descent-n}, while the corresponding descent estimates with respect to the variables $\mathbf{u}$ and $\mathbf{q}$ are established in \cite{liu2024fast}.
\end{proof}

\begin{lemma}
Let $\{(\mathbf{u}^{k+1},\mathbf{n}^{k+1},\mathbf{q}^{k+1})\}_{k=0}^{+\infty}$ be a sequence generated by the proposed MMAMM. Then, there exists a positive constant $C$ independent of the index $k$ such that
\begin{equation*}
\max_k\left\{\|\mathbf{u}^{k+1}\|,\|\mathbf{n}^{k+1}\|,\|\mathbf{q}^{k+1}\|\right\} \leq C.
\end{equation*}
\end{lemma}
\begin{proof}
$\mathbf{n}$ is bounded, and the functional $E_{\alpha,\mathbb{I}}$ is coercive with respect to $(\mathbf{u},\mathbf{q})$. Invoking Lemma~\ref{lem:energy-descent} yields the uniform boundedness of $\|\mathbf{u}^{k+1}\|$, $\|\mathbf{n}^{k+1}\|$ and $\|\mathbf{q}^{k+1}\|$.
\end{proof}

Having established the boundedness of the iteration sequence, we next derive an upper bound for the subgradient of the objective function along the iterates, which plays a central role in the subsequent convergence analysis.
\begin{lemma}
There exists $\mathbf{g}^{k+1} := (\mathbf{g}_\mathbf{u}^{k+1},\mathbf{g}_\mathbf{n}^{k+1},\mathbf{g}_\mathbf{q}^{k+1})$ with
\begin{align*}
\mathbf{g}_\mathbf{u}^{k+1} & = \nabla_\mathbf{u} E_{\alpha}(\mathbf{u}^{k+1},\mathbf{n}^{k+1},\mathbf{q}^{k+1}), \\
\mathbf{g}_\mathbf{n}^{k+1} & \in \nabla_\mathbf{n} E_{\alpha}(\mathbf{u}^{k+1},\mathbf{n}^{k+1},\mathbf{q}^{k+1}) + \partial_\mathbf{n} \mathbb{I}_{\mathcal{S}_{j,+}}(\mathbf{n}^{k+1}), \\
\mathbf{g}_\mathbf{q}^{k+1} & \in \nabla_\mathbf{q} E_{\alpha}(\mathbf{u}^{k+1},\mathbf{n}^{k+1},\mathbf{q}^{k+1}) + \partial_\mathbf{q} \mathbb{I}_{\mathcal{R}_\beta}(\mathbf{q}^{k+1}),
\end{align*}
such that
\begin{align*}
\|\mathbf{g}^{k+1}\| &\leq \|\mathbf{g}_\mathbf{u}^{k+1}\| + \|\mathbf{g}_\mathbf{n}^{k+1}\| + \|\mathbf{g}_\mathbf{q}^{k+1}\| \\
& \leq \hat{C}(\|\mathbf{q}^{k+1}-\mathbf{q}^k\| + \|\mathbf{n}^{k+1}-\mathbf{n}^k\|).
\end{align*}
where $\hat{C}$ is a constant independent of the index $k$.
\end{lemma}

\begin{proof}
For $\mathbf{u}$-subproblem,
\begin{align*}
\| \mathbf{g}_\mathbf{u}^{k+1} \|
&= \big\| \nabla_\mathbf{u} E_\alpha (\mathbf{u}^{k+1}, \mathbf{n}^{k+1}, \mathbf{q}^{k+1}) \big\| \\
&\leq \big\| \nabla_\mathbf{u} E_\alpha (\mathbf{u}^{k+1}, \mathbf{n}^{k+1}, \mathbf{q}^{k+1}) - \nabla_\mathbf{u} E_\alpha (\mathbf{u}^{k+1}, \mathbf{n}^{k}, \mathbf{q}^{k+1}) \big\| \\
&\quad + \big\| \nabla_\mathbf{u} E_\alpha (\mathbf{u}^{k+1}, \mathbf{n}^{k}, \mathbf{q}^{k+1}) - \nabla_\mathbf{u} E_\alpha (\mathbf{u}^{k+1}, \mathbf{n}^{k}, \mathbf{q}^{k}) \big\| \\
&\quad + \big\| \nabla_\mathbf{u} E_\alpha (\mathbf{u}^{k+1}, \mathbf{n}^{k}, \mathbf{q}^{k}) \big\| \\
&\leq C_1 (\| \mathbf{n}^{k+1} - \mathbf{n}^{k} \| + \| \mathbf{q}^{k+1} - \mathbf{q}^{k} \| ),
\end{align*}
where $C_1$ is a constant independent of $k$. For $\xi^{k+1}\in\partial_\mathbf{n} \mathbb{I}_{\mathcal{S}_{j,+}}(\mathbf{n}^{k+1})$, using the optimality condition of the $\mathbf{n}$-subproblem, we have
\begin{align*}
0 &= \nabla_{\mathbf{n}} Q(\mathbf n^{k+1} \mid \mathbf n^k) + \xi^{k+1} \\
& = \nabla_{\mathbf{n}}\phi(\mathbf{n}^k) + \nabla_\mathbf{n}h^k(\mathbf{n}^k) + (H^k+\mu\mathbf{I})(\mathbf{n}^{k+1}-\mathbf{n}^k) + \xi^{k+1}\\
& = \nabla_\mathbf{n} E_\alpha (\mathbf{u}^{k+1}, \mathbf{n}^k, \mathbf{q}^{k}) + (H^k+\mu\mathbf{I})(\mathbf{n}^{k+1}-\mathbf{n}^k) + \xi^{k+1},
\end{align*}
where the last equation is obtained by using $\nabla_\mathbf{n} E_\alpha (\mathbf{u}^{k+1}, \mathbf{n}, \mathbf{q}^{k}) = \nabla_{\mathbf{n}}\phi(\mathbf{n}) + \nabla_\mathbf{n}h^k(\mathbf{n})$. Then, we have
\begin{align*}
\|\mathbf{g}_\mathbf{n}^{k+1}\|
&= \|\nabla_\mathbf{n} E_{\alpha}(\mathbf{u}^{k+1},\mathbf{n}^{k+1},\mathbf{q}^{k+1}) + \xi^{k+1}\| \\
& \leq \|\nabla_\mathbf{n} E_{\alpha}(\mathbf{u}^{k+1},\mathbf{n}^{k+1},\mathbf{q}^{k+1}) - \nabla_\mathbf{n} E_\alpha (\mathbf{u}^{k+1}, \mathbf{n}^{k+1}, \mathbf{q}^{k})\| \\
& \quad + \|\nabla_\mathbf{n} E_{\alpha}(\mathbf{u}^{k+1},\mathbf{n}^{k+1},\mathbf{q}^k) - \nabla_\mathbf{n} E_\alpha (\mathbf{u}^{k+1}, \mathbf{n}^k, \mathbf{q}^{k})\| \\
& \quad + \|(H^k+\mu\mathbf{I})(\mathbf{n}^{k+1}-\mathbf{n}^k)\| \\
& \leq C_2(\|\mathbf{n}^{k+1}-\mathbf{n}^k\| + \|\mathbf{q}^{k+1}-\mathbf{q}^k\|),
\end{align*}
where $C_2$ is a constant independent of $k$. For the $\mathbf{q}$-subproblem, we have
\[
0 \in \nabla_\mathbf{q} E_\alpha (\mathbf{u}^{k+1}, \mathbf{n}^{k+1}, \mathbf{q}^{k+1}) + \partial_\mathbf{q} \mathbb{I}_{\mathcal{R}_\beta}(\mathbf{q}^{k+1}).
\]
Then, there exists a $\zeta^{k+1} \in \partial_\mathbf{q} \mathbb{I}_{\mathcal{R}_\beta}(\mathbf{q}^{k+1})$ such that
\[
\|\mathbf{g}_\mathbf{q}^{k+1}\| = \|\nabla_\mathbf{q} E_\alpha (\mathbf{u}^{k+1}, \mathbf{n}^{k+1}, \mathbf{q}^{k+1}) + \zeta^{k+1}\| = 0.
\]
This proves the lemma.
\end{proof}

Building on these foundational results and leveraging the KL property of the objective function, we establish the main convergence theorem. The detailed proof follows the established framework in \cite{bolte2014proximal, liu2024fast}.
\begin{theorem}
Let $\{(\mathbf{u}^{k+1},\mathbf{n}^{k+1},\mathbf{q}^{k+1})\}_{k=0}^{+\infty}$ be a sequence generated by the proposed MMAMM. Then
\begin{enumerate}
    \item The sequence $\{(\mathbf{u}^{k+1},\mathbf{n}^{k+1},\mathbf{q}^{k+1})\}_{k=0}^{+\infty}$ has a finite length, i.e.,
    \[
    \sum_{k=0}^{+\infty} \| (\mathbf{u}^{k+1},\mathbf{n}^{k+1},\mathbf{q}^{k+1}) - (\mathbf{u}^k,\mathbf{n}^k,\mathbf{q}^k)\| < +\infty.
    \]
    \item The sequence $\{(\mathbf{u}^{k+1},\mathbf{n}^{k+1},\mathbf{q}^{k+1})\}_{k=0}^{+\infty}$ converges to a critical point of the quadratic-penalty problem $E_{\alpha,\mathbb{I}}$.
\end{enumerate}
\end{theorem}

\begin{remark}
We finally comment on the feasibility of the bilinear constraints in the penalized formulation. For the iterates generated by MMAMM, define
\[
\mathbf r_x^k=\mathbf D_x^+\mathbf u^k-\mathbf q^k\odot \mathbf n_1^k,\quad
\mathbf r_y^k=\mathbf D_y^+\mathbf u^k-\mathbf q^k\odot \mathbf n_2^k,\quad
\mathbf r_\beta^k=\boldsymbol\beta-\mathbf q^k\odot \mathbf n_3^k .
\]
These quantities measure the violation of the three equality constraints in the bilinear decomposition. Since the energy is nonincreasing and the penalty term is nonnegative, we have
\[
\frac{\alpha}{2}
\left(
\|\mathbf r_x^k\|^2
+
\|\mathbf r_y^k\|^2
+
\|\mathbf r_\beta^k\|^2
\right)
\leq
E_{\alpha,\mathbb I}(\mathbf u^k,\mathbf n^k,\mathbf q^k)
<
E_{\alpha,\mathbb I}(\mathbf u^0,\mathbf n^0,\mathbf q^0).
\]
Consequently,
\[
\left(
\|\mathbf r_x^k\|^2
+
\|\mathbf r_y^k\|^2
+
\|\mathbf r_\beta^k\|^2
\right)^{1/2}
<
\sqrt{
\frac{
2E_{\alpha,\mathbb I}(\mathbf u^0,\mathbf n^0,\mathbf q^0)
}{\alpha}
}.
\]
Thus, for a fixed penalty parameter $\alpha$, the feasibility violation is uniformly controlled along the whole sequence. In particular, increasing $\alpha$ enforces the bilinear constraints more strongly.
\end{remark}

%% file: sections/sec-experiments.tex
\section{Numerical experiments} \label{sec:experiments}
In this section, we present a series of numerical experiments to evaluate the effectiveness and efficiency of the proposed MMAMM algorithms (Algorithm \ref{algorithm}) for solving the TSGV-based bilinear decomposition model \eqref{unconstrained_discrete_bilinear_model}. We first compare the proposed MMAMM algorithms with the IOS algorithm for the TSGV model (IOS--TSGV) in terms of edge and corner preservation and the convergence behavior of the relative error. Subsequently, we compare the proposed method with several representative variational denoising approaches, including IOS--TSGV and the hybrid alternating minimization algorithm for the Euler elastica model (HALM--EE). Finally, we apply the proposed method to the NLOS imaging problem to further demonstrate its applicability and effectiveness in practical imaging tasks. All image denoising experiments are performed on a laptop running Windows 11 (64-bit), 32GB of RAM, an Intel Core i9-14900HX (2.20 GHz), and MATLAB R2023a. For the NLOS imaging problem, all experiments are further accelerated using an NVIDIA GeForce RTX 5060 Laptop GPU.

\subsection{Experiment settings}

\begin{figure}[tbp]
	\centering
	\subfigure[\shortstack{TestImg1\\($256\times256$)}]{
    \includegraphics[scale=0.45]{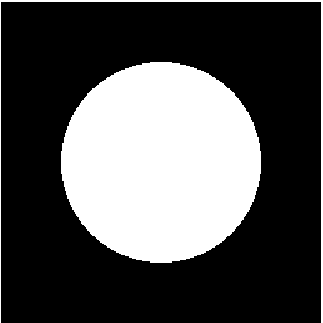}
	}
	\subfigure[\shortstack{TestImg2\\($256\times256$)}]{
    \includegraphics[scale=0.45]{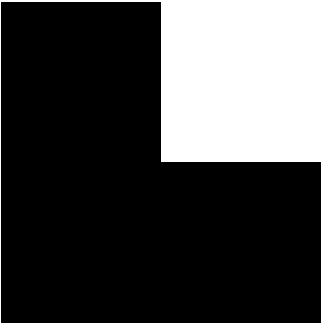}
	}
	\subfigure[\shortstack{TestImg3\\($256\times256$)}]{
    \includegraphics[scale=0.45]{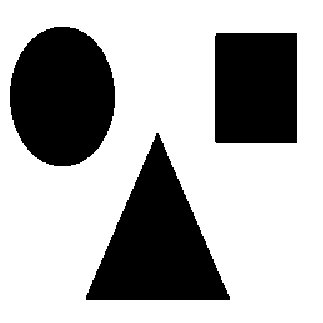}
	}\\
    \subfigure[\shortstack{TestImg4\\($60\times60$)}]{
    \includegraphics[scale=1.9]{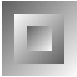}
	}
	\subfigure[\shortstack{TestImg5\\($256\times256$)}]{
    \includegraphics[scale=0.45]{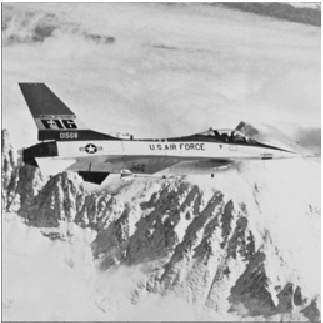}
	}
	\subfigure[\shortstack{TestImg6\\($256\times256$)}]{
    \includegraphics[scale=0.45]{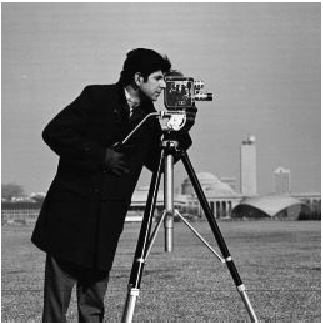}
    }
    \subfigure[\shortstack{TestImg7\\($321\times481$)}]{
    \includegraphics[scale=0.36]{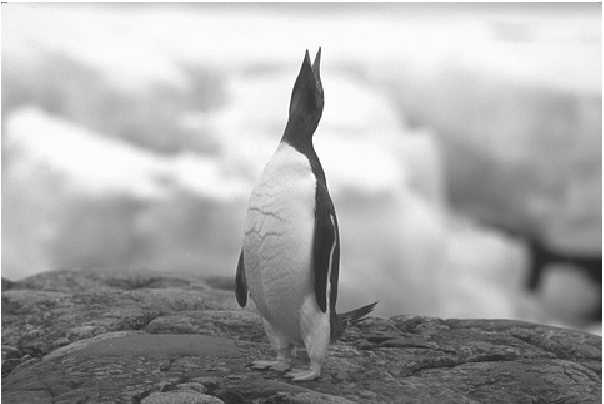}
    }\\
    \subfigure[\shortstack{TestImg8\\($60\times60$)}]{
    \includegraphics[scale=1.9]{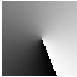}
	}
	\subfigure[\shortstack{TestImg9\\($256\times256$)}]{
    \includegraphics[scale=0.45]{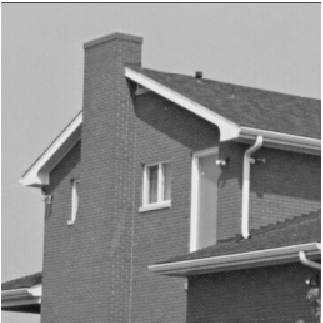}
	}
	\subfigure[\shortstack{TestImg10\\($256\times256$)}]{
    \includegraphics[scale=0.45]{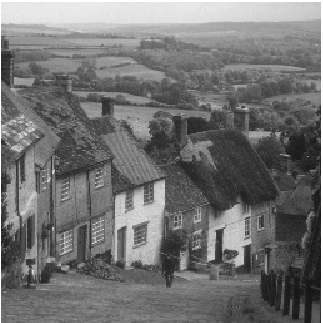}
    }
    \subfigure[\shortstack{TestImg11\\($216\times332$)}]{
    \includegraphics[scale=0.535]{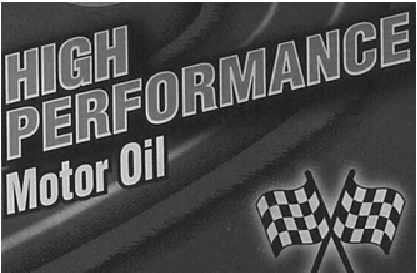}
    }
\caption{Original test images used in the Gaussian denoising experiments. In numerical tests, Gaussian noise with mean $0$ and variance $0.0025$ is added to TestImg$1$--TestImg$7$, while Gaussian noise with mean $0$ and variance $0.005$ is added to TestImg$8$--TestImg$11$.} \label{figure_denoising_image}
\end{figure}

Different parameter settings are adopted for different scaling functions $\psi$ in Algorithm \ref{algorithm}. In particular, we set $\epsilon = 0.001$ for Algorithm \ref{algorithm} with $\psi_1$ (denoted as MMAMM--$\psi_1$); we set $\epsilon = 0.01$ for Algorithm \ref{algorithm} with $\psi_2$ (denoted as MMAMM--$\psi_2$). Moreover, we set $b = 0.001$ and choose $\mu = \frac{8}{\sqrt{\epsilon}}+0.1$, so that the condition $\mu>L_\phi$ required in the convergence analysis is satisfied. The algorithm is terminated once the relative error
\begin{equation} \label{relative_error}
\frac{\|u^k-u^{k-1}\|}{\|u^{k-1}\|}\leq \mathrm{Tol}
\end{equation}
with $\mathrm{Tol}=1\times 10^{-4}$ is satisfied, or when the maximum number of iterations $\mathrm{Maxit}=500$ is reached. The parameter settings for all competing methods are set to the default values provided in their original implementations, and all algorithmic parameters are further tuned to achieve the best PSNR performance. The intensity values of both the observed image $f$ and the reconstructed image $u$ are normalized to the interval $[0,1]$. Figure \ref{figure_denoising_image} shows the test images used in the denoising experiments, including binary images, synthetic images, and real images with different resolutions, where the corresponding resolutions are displayed below each image.

\begin{figure}[tbp]
    \centering
    \subfigure[MMAMM--$\psi_1$ on TestImg4]{
    \includegraphics[scale=0.38]{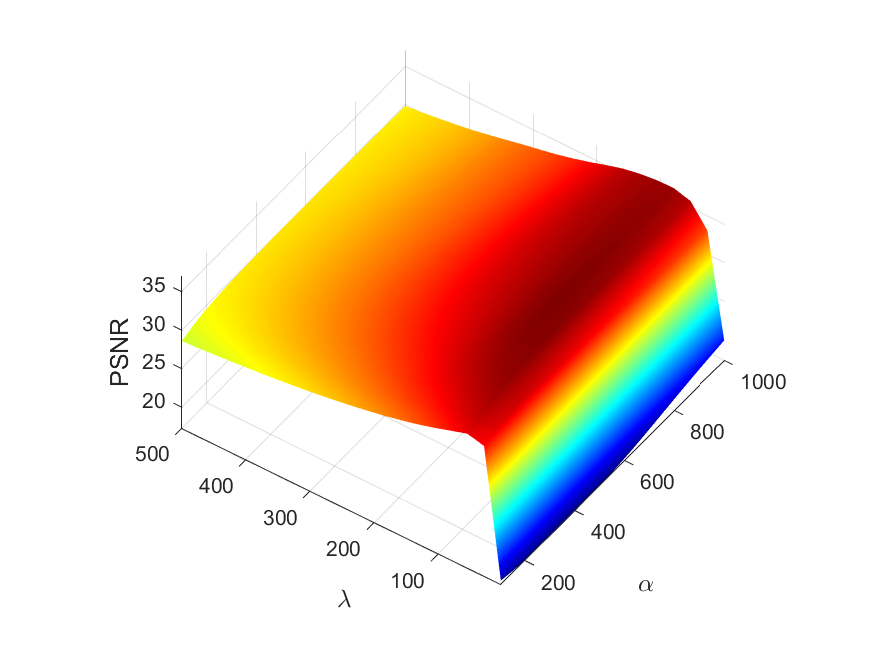}
    }
    \subfigure[MMAMM--$\psi_1$ on TestImg8]{
    \includegraphics[scale=0.38]{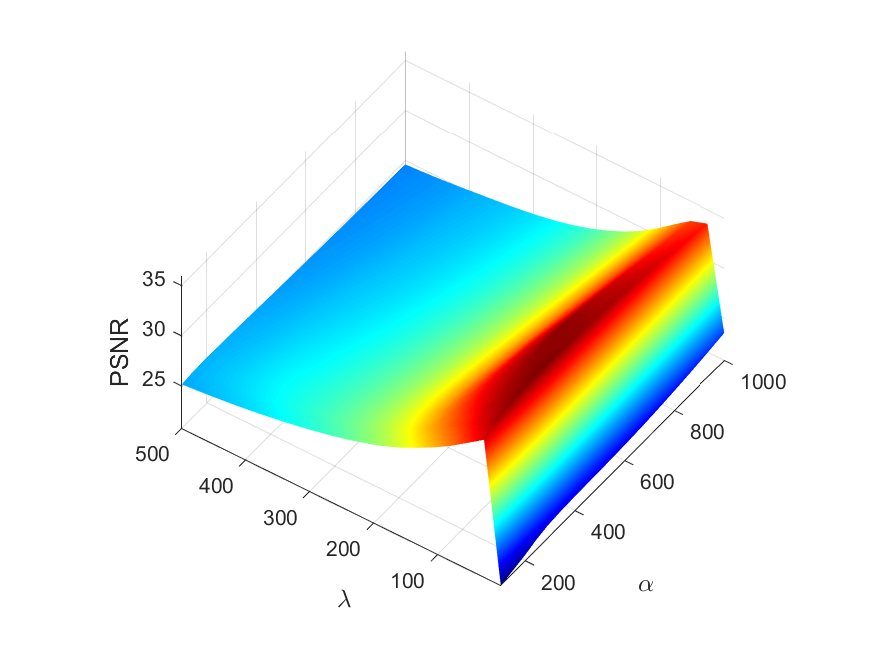}
    }\\
    \subfigure[MMAMM--$\psi_2$ on TestImg4]{
    \includegraphics[scale=0.38]{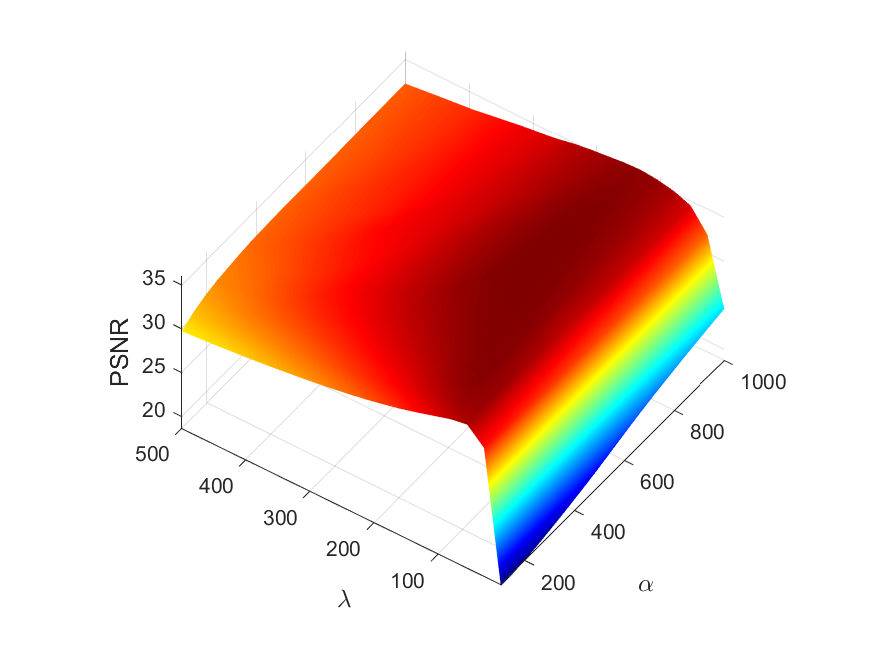}
    }
    \subfigure[MMAMM--$\psi_2$ on TestImg8]{
    \includegraphics[scale=0.38]{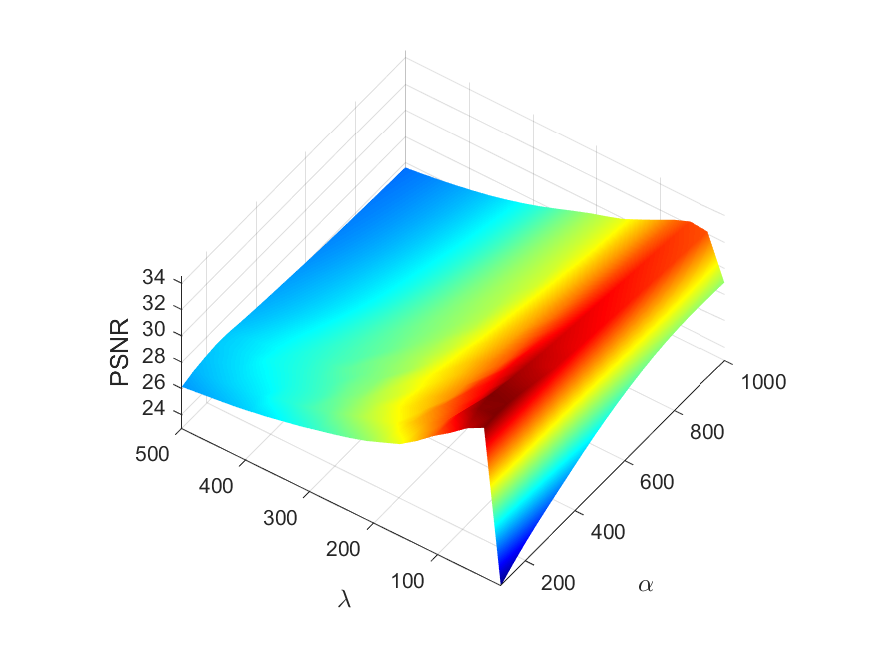}
    }
\caption{Performance of the proposed MMAMM algorithms for $\lambda \in [1,500]$ and $\alpha \in [100,1000]$ on TestImg$4$ and TestImg$8$ with different noise levels.} \label{figure_lambda_alpha}
\end{figure}

In the proposed MMAMM algorithms, two parameters must be manually adjusted: the model parameter $\lambda$ and the penalty parameter $\alpha$. In the following, we numerically investigate suitable parameter ranges for these two parameters using TestImg4 and TestImg8 as representative examples. Figure \ref{figure_lambda_alpha} presents the PSNR values achieved by MMAMM--$\psi_1$ and MMAMM--$\psi_2$ for $\lambda \in [1,500]$ and $\alpha \in [100,1000]$. Empirically, for most Gaussian denoising experiments, the proposed MMAMM algorithms perform well for $\lambda\in[10,300]$ and $\alpha\in[100,800]$.

\subsection{Advantages of MMAMM over the IOS algorithm}

In this subsection, we compare the proposed MMAMM algorithm with the IOS algorithm for solving the TSGV model. Since the TSGV model is capable of preserving edge and corner contrasts, we first use TestImg1--TestImg3 to evaluate the two algorithms' ability to preserve these structures. Figure \ref{figure_edge_corner_preservation} presents the noisy images and the corresponding residual images, where the residual images are computed by $\mathbf{f}-\mathbf{u}+0.4$. From the enlarged regions in Figure \ref{figure_edge_corner_preservation}, it can be observed that the residual images obtained by MMAMM--$\psi_1$ contain less edge information than those obtained by IOS--TSGV and MMAMM--$\psi_2$, indicating that MMAMM--$\psi_1$ can better exploit the edge- and corner-preserving capability of the TSGV model. Table \ref{table_edge_corner_preservation} reports the corresponding quantitative results, including the PSNR values, the number of iterations, the running times, and the average time per iteration (Time/$\#$Iter) of the algorithms. It shows that although the MMAMM algorithms require more iterations and longer overall runtimes, the average computational time per iteration is significantly smaller than that of IOS--TSGV (approximately $1/6$). Moreover, the MMAMM algorithms achieve substantially higher PSNR values than IOS--TSGV. Compared with MMAMM--$\psi_1$, MMAMM--$\psi_2$ attains higher PSNR values for the restored images. Taken together with the previous observations, these results suggest that for binary images with simple edge structures such as TestImg1--TestImg3, MMAMM--$\psi_1$ provides better edge preservation, whereas MMAMM--$\psi_2$ achieves the highest PSNR.

\begin{figure}[t]
    \centering
    \subfigure[Noisy image]{
        \begin{minipage}[b]{0.17\linewidth}
        \centering
        \includegraphics[scale=0.407]{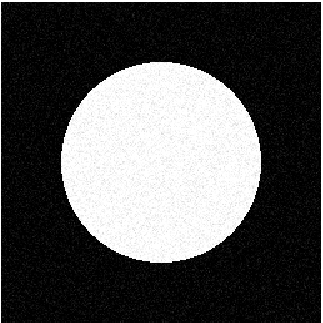}\\
        \vspace{0.02cm}
        \includegraphics[scale=0.407]{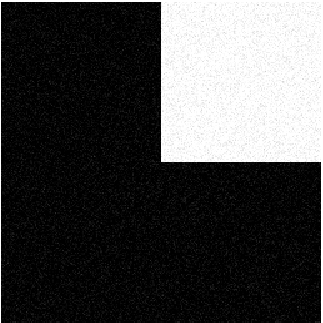}\\
        \vspace{0.02cm}
        \includegraphics[scale=0.407]{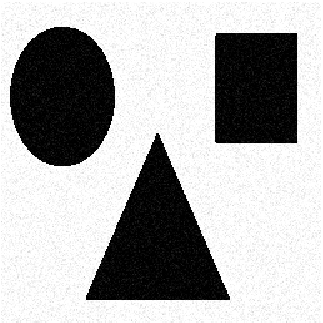}\\
        \vspace{0.02cm}
        \end{minipage}
        }
    \subfigure[Noise]{
        \begin{minipage}[b]{0.17\linewidth}
        \centering
        \includegraphics[scale=0.2]{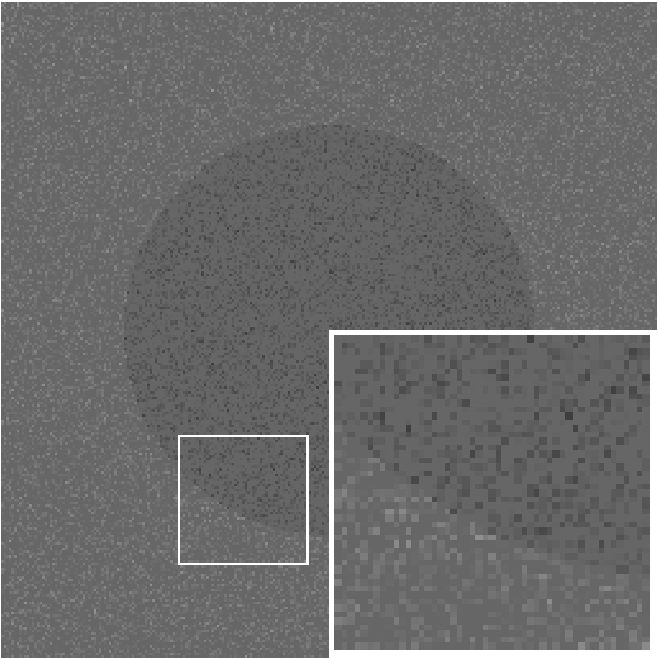}\\
        \vspace{0.02cm}
        \includegraphics[scale=0.2]{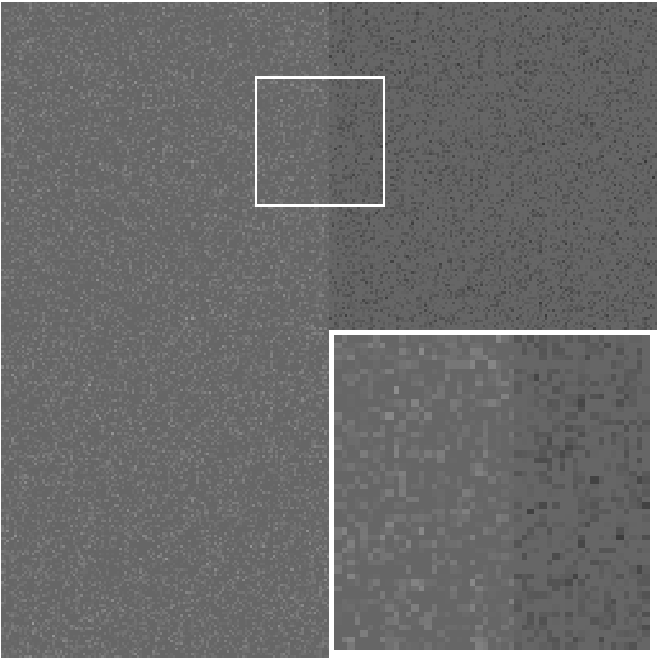}\\
        \vspace{0.02cm}
        \includegraphics[scale=0.2]{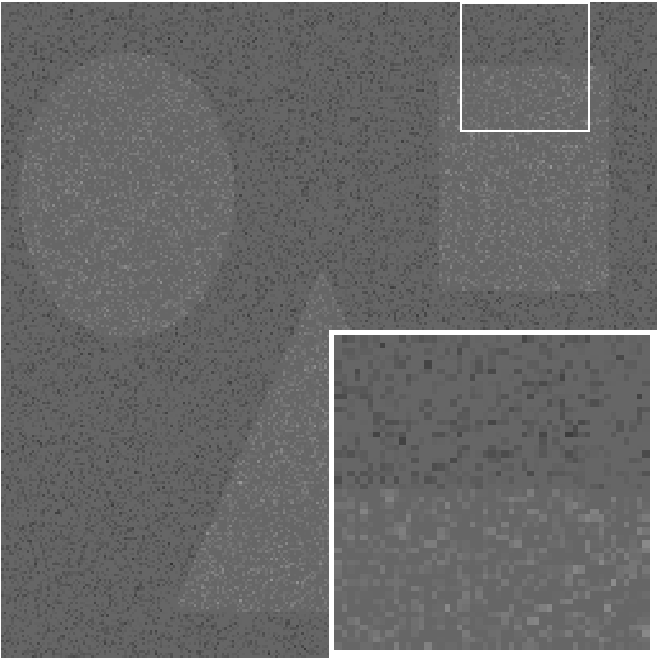}\\
        \vspace{0.02cm}
        \end{minipage}
        }
    \subfigure[IOS--TSGV]{
        \begin{minipage}[b]{0.17\linewidth}
        \centering
        \includegraphics[scale=0.2]{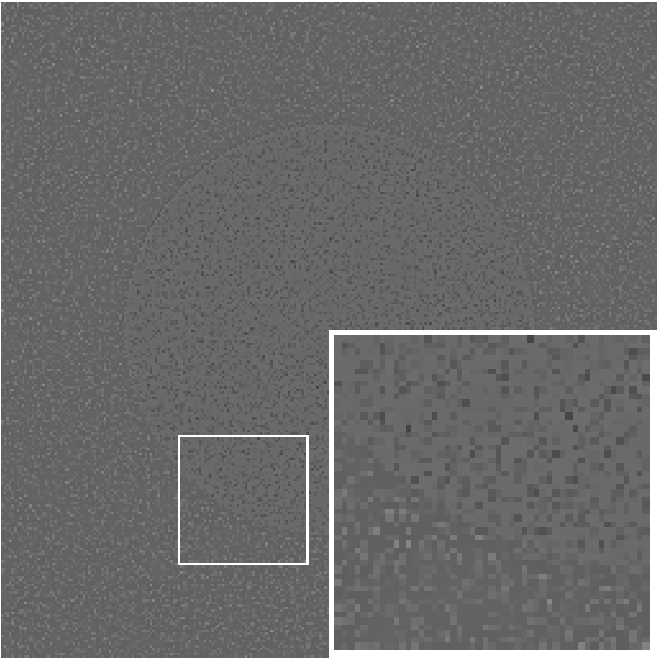}\\
        \vspace{0.02cm}
        \includegraphics[scale=0.2]{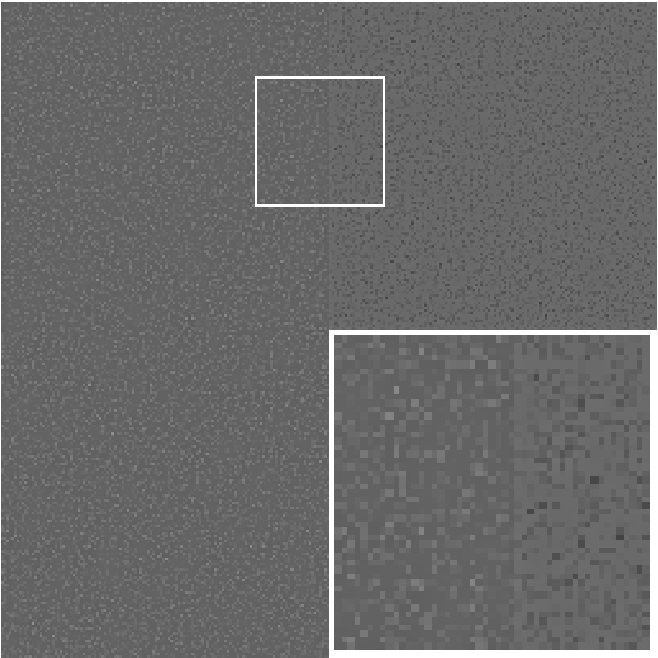}\\
        \vspace{0.02cm}
        \includegraphics[scale=0.2]{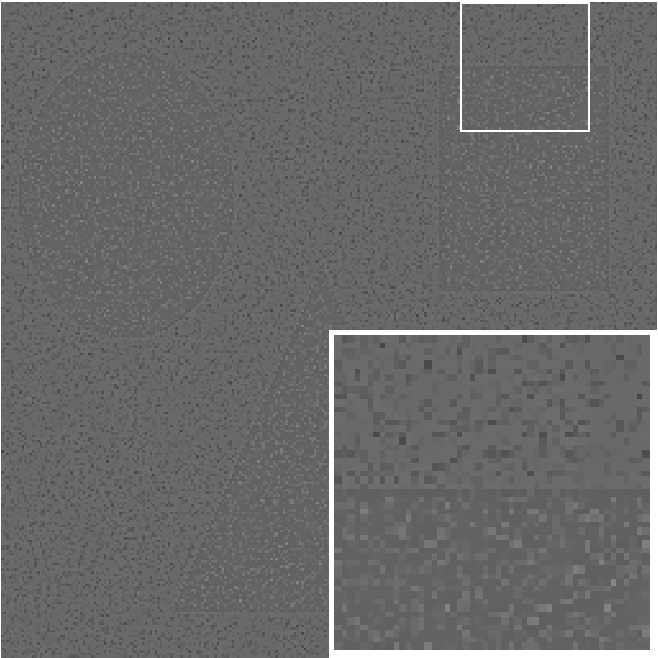}\\
        \vspace{0.02cm}
        \end{minipage}
        }
        \subfigure[MMAMM--$\psi_1$]{
        \begin{minipage}[b]{0.17\linewidth}
        \centering
        \includegraphics[scale=0.2]{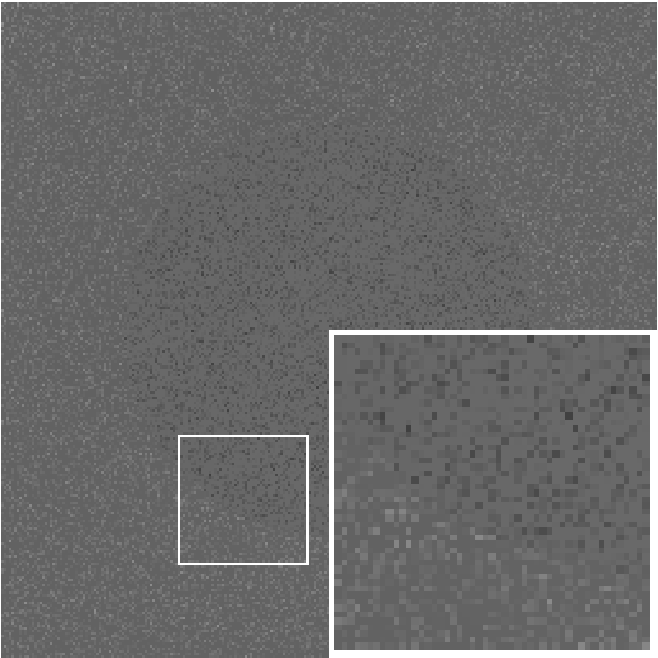}\\
        \vspace{0.02cm}
        \includegraphics[scale=0.2]{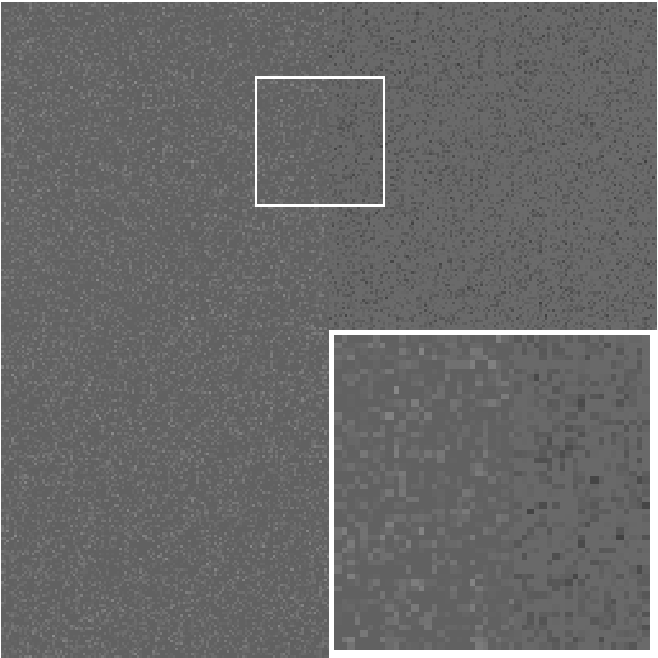}\\
        \vspace{0.02cm}
        \includegraphics[scale=0.2]{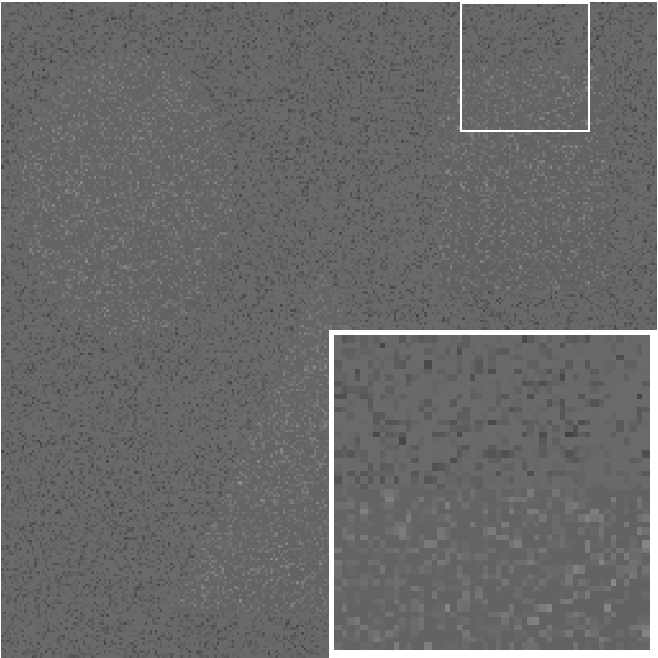}\\
        \vspace{0.02cm}
        \end{minipage}
        }
        \subfigure[MMAMM--$\psi_2$]{
        \begin{minipage}[b]{0.17\linewidth}
        \centering
        \includegraphics[scale=0.2]{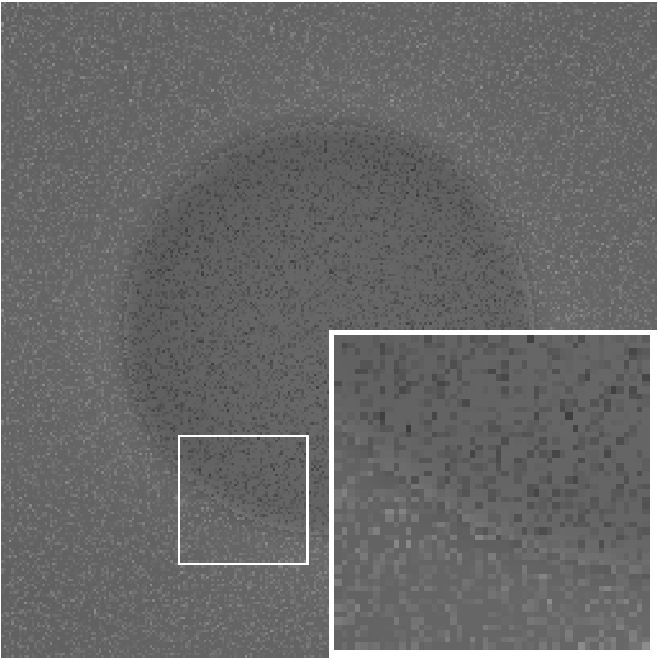}\\
        \vspace{0.02cm}
        \includegraphics[scale=0.2]{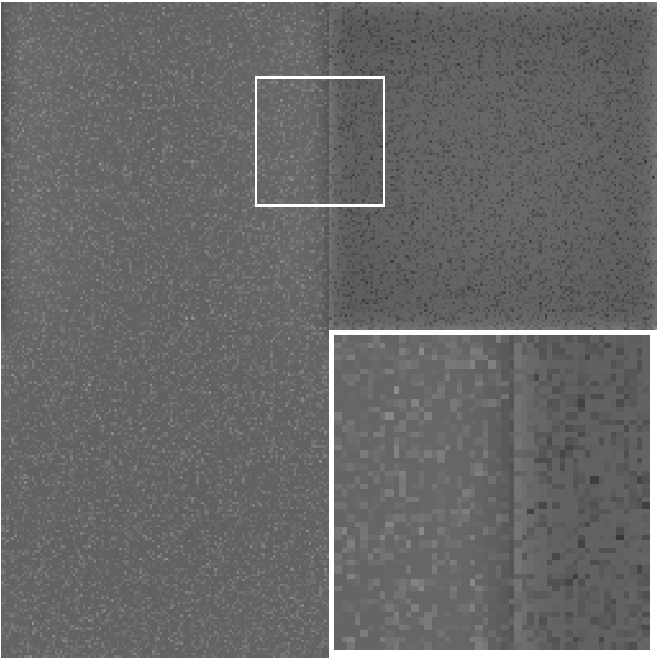}\\
        \vspace{0.02cm}
        \includegraphics[scale=0.2]{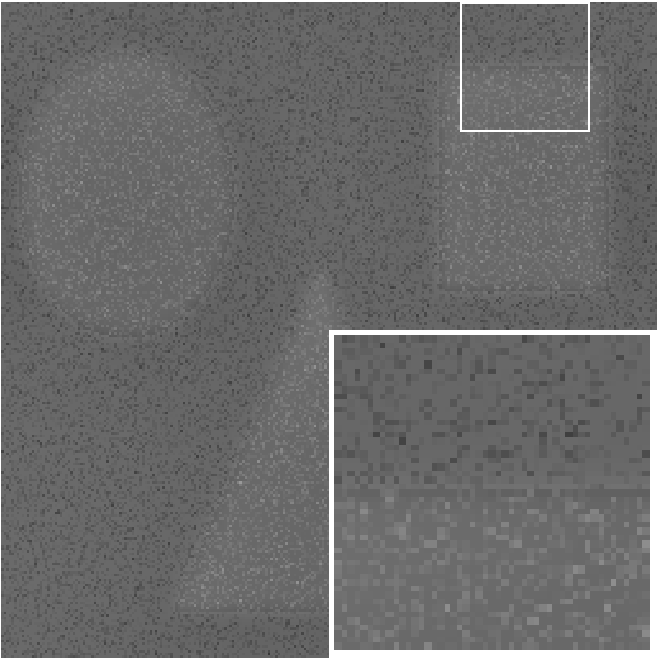}\\
        \vspace{0.02cm}
        \end{minipage}
        }
\caption{Residual images for Gaussian denoising on TestImg$1$--TestImg$3$. (a) Noisy images; (b) residual images of the noisy observations; (c) residual images obtained by IOS--TSGV; (d) residual images obtained by MMAMM--$\psi_1$; (e) residual images obtained by MMAMM--$\psi_2$. The residual images are computed by $\mathbf{f}-\mathbf{u}+0.4$.} \label{figure_edge_corner_preservation}
\end{figure}
\begin{table}[tbp]
\centering
\small{
\caption{Quantitative results of IOS--TSGV and the proposed MMAMM algorithms, including the PSNR values, number of iterations ($\#$Iter), running times (Time) and average time
per iteration (Time/$\#$Iter).} \label{table_edge_corner_preservation}
\setlength{\tabcolsep}{0.8mm}{
\vspace{0.1cm}
\begin{tabular}{ccccccccccc}
\toprule
                & \multicolumn{3}{c}{TestImg1} & \multicolumn{3}{c}{TestImg2} & \multicolumn{3}{c}{TestImg3} & \multirow{2}{*}{Time/\#Iter} \\
                \cmidrule(lr){2-4} \cmidrule(lr){5-7} \cmidrule(lr){8-10}
                & PSNR    & $\#$Iter   & Time  & PSNR    & $\#$Iter   & Time  & PSNR    & $\#$Iter   & Time  &                              \\ \midrule
IOS--TSGV       & 33.36   & 54         & 0.5   & 33.45   & 27         & 0.3   & 33.18   & 54         & 0.5   & 0.029                        \\
MMAMM--$\psi_1$ & 36.11   & 129        & 0.6   & 34.41   & 129        & 0.6   & 36.38   & 111        & 0.5   & 0.005                        \\
MMAMM--$\psi_2$ & 37.85   & 186        & 0.9   & 35.78   & 258        & 1.2   & 37.95   & 142        & 0.7   & 0.005                        \\ \bottomrule
\end{tabular}}}
\end{table}

We further compare the denoising performance of different algorithms in smooth regions. The restored image surfaces are shown in Figure~\ref{figure_surface}. IOS--TSGV removes a large amount of noise, but its reconstructed surfaces still exhibit visible local fluctuations. In comparison, the surfaces obtained by the proposed MMAMM methods are smoother and more regular, indicating that the proposed method can effectively suppress noise. In particular, MMAMM--$\psi_1$ yields a visually smoother surface and better removes small oscillatory artifacts. For MMAMM--$\psi_2$, a few isolated points still contain visible residual noise, suggesting that the sphere-constrained model may be slightly less effective in suppressing certain local perturbations in these examples.
\begin{figure}[t]
    \centering
    \subfigure[Noisy image]{
        \begin{minipage}[b]{0.2\linewidth}
        \centering
        \includegraphics[scale=0.17]{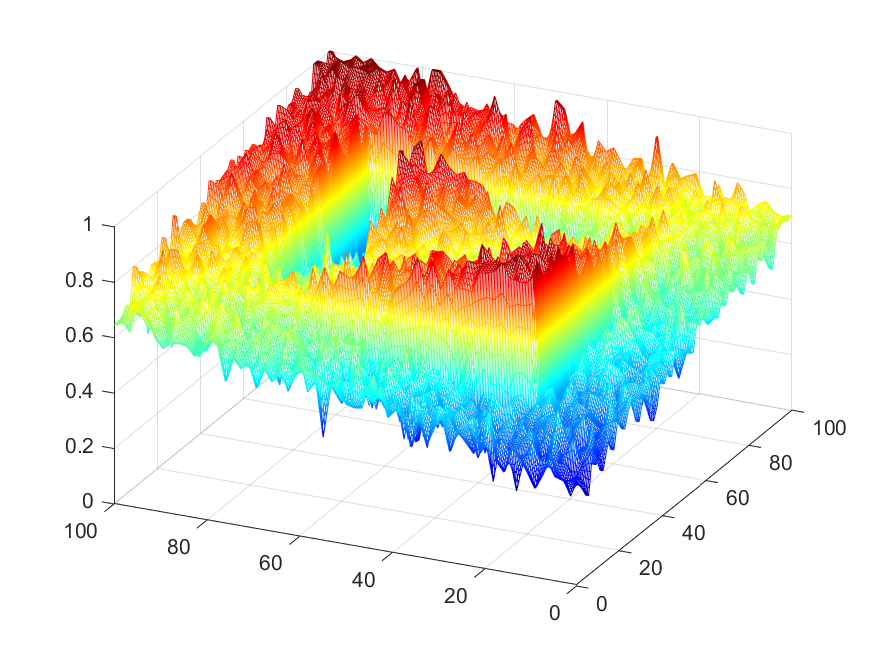}\\
        \vspace{0.02cm}
        \includegraphics[scale=0.17]{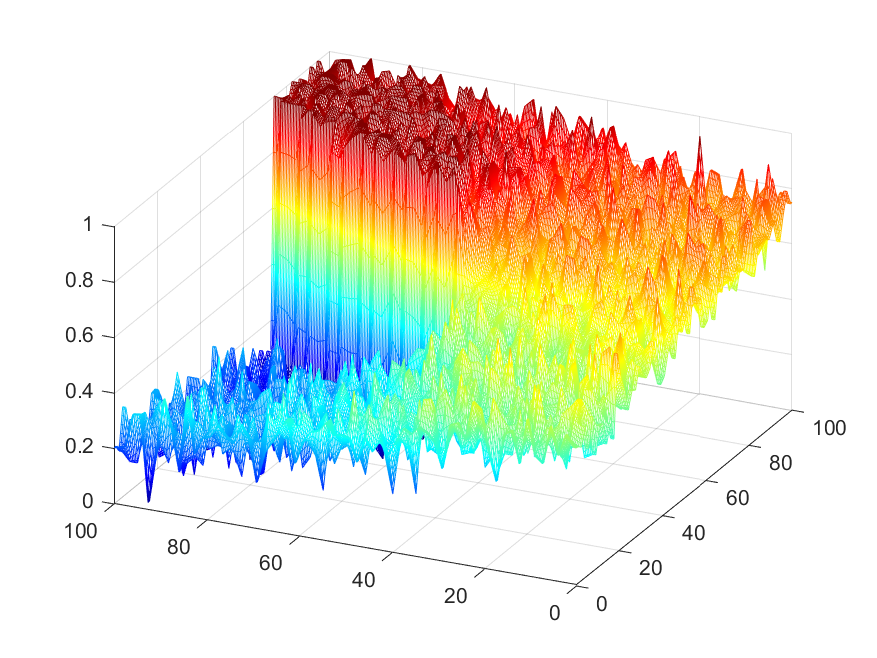}\\
        \vspace{0.02cm}
        \end{minipage}
        }
    \subfigure[IOS--TSGV]{
        \begin{minipage}[b]{0.2\linewidth}
        \centering
        \includegraphics[scale=0.17]{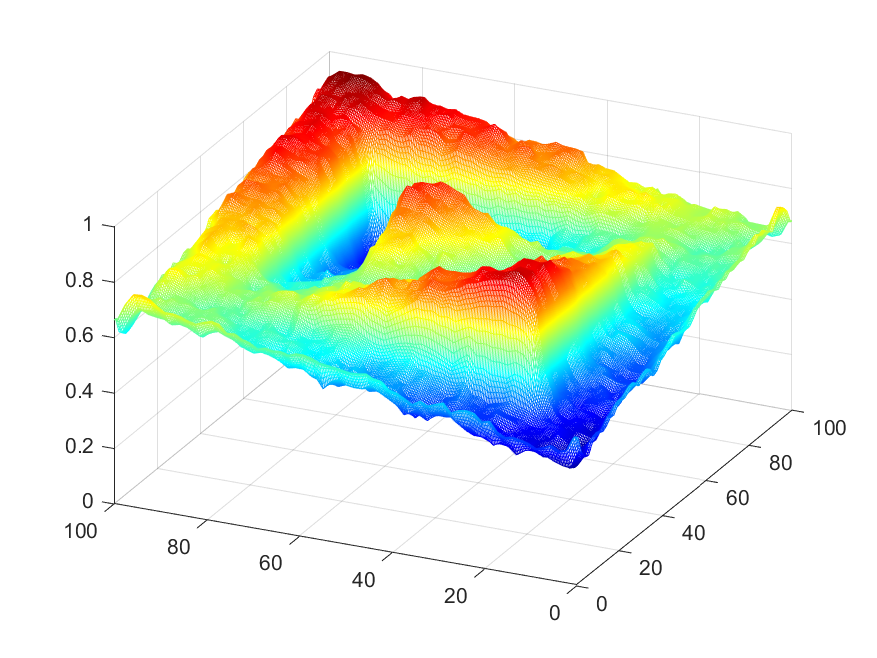}\\
        \vspace{0.02cm}
        \includegraphics[scale=0.17]{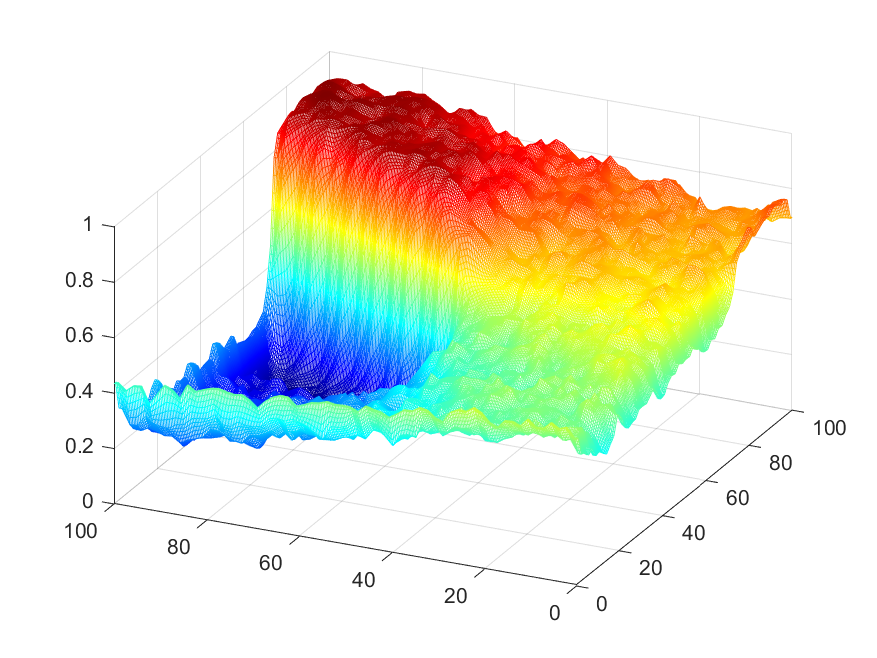}\\
        \vspace{0.02cm}
        \end{minipage}
        }
    \subfigure[MMAMM--$\psi_1$]{
        \begin{minipage}[b]{0.2\linewidth}
        \centering
        \includegraphics[scale=0.17]{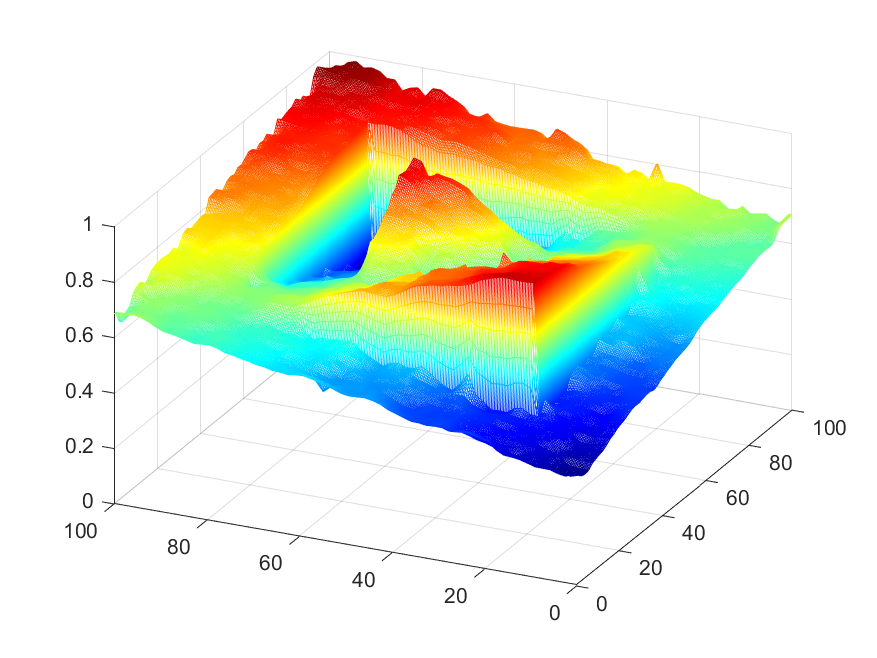}\\
        \vspace{0.02cm}
        \includegraphics[scale=0.17]{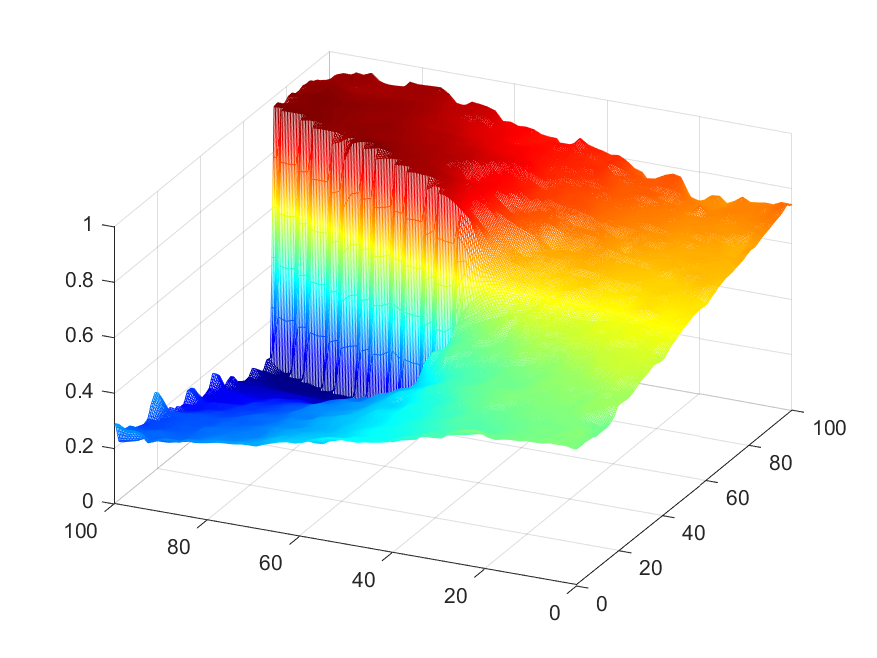}\\
        \vspace{0.02cm}
        \end{minipage}
        }
    \subfigure[MMAMM--$\psi_2$]{
        \begin{minipage}[b]{0.2\linewidth}
        \centering
        \includegraphics[scale=0.17]{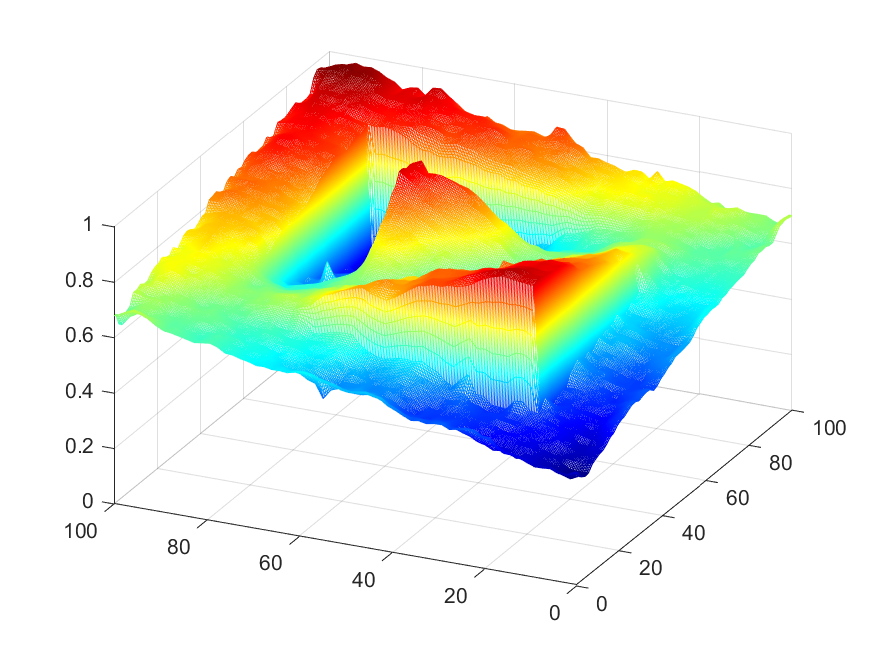}\\
        \vspace{0.02cm}
        \includegraphics[scale=0.17]{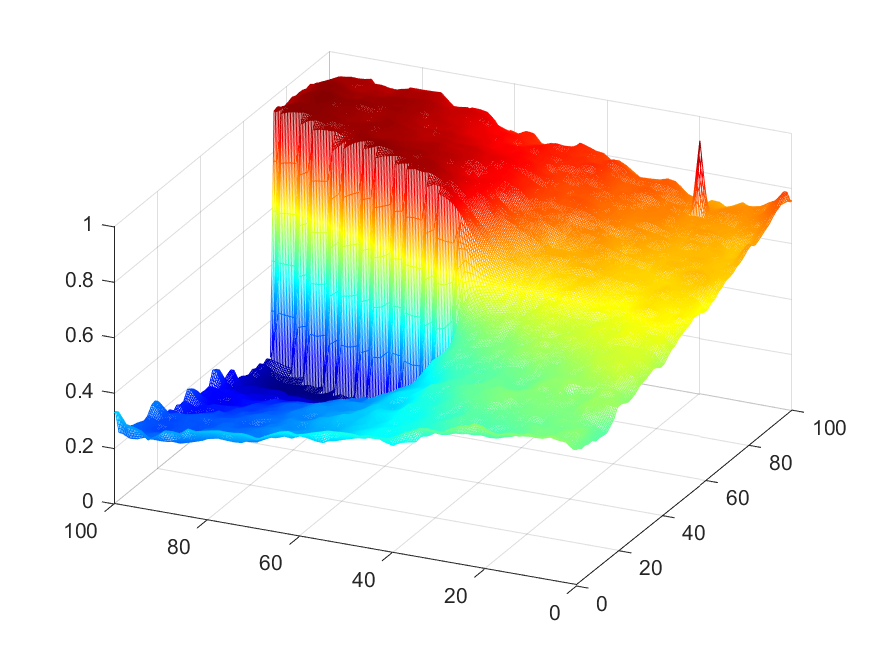}\\
        \vspace{0.02cm}
        \end{minipage}
        }
\caption{Image surface comparison for Gaussian denoising on TestImg$4$ and TestImg$8$. (a) Noisy image surfaces; (b) restored image surfaces by IOS--TSGV; (c) restored image surfaces by MMAMM--$\psi_1$; (d) restored image surfaces by MMAMM--$\psi_2$.} \label{figure_surface}
\end{figure}
\begin{figure}[tbp]
    \centering
    \subfigure[TestImg4]{
    \includegraphics[scale=0.407]{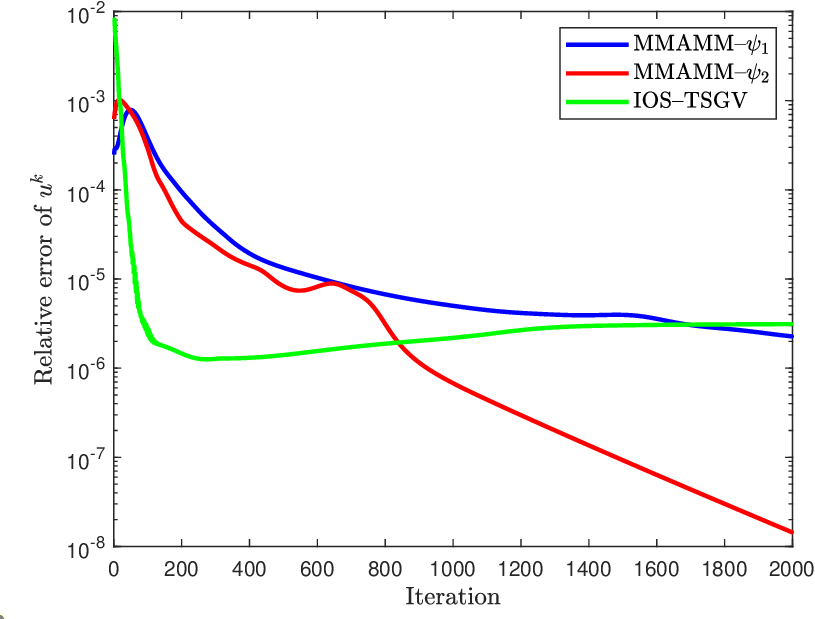}
    }
    \subfigure[TestImg8]{
    \includegraphics[scale=0.407]{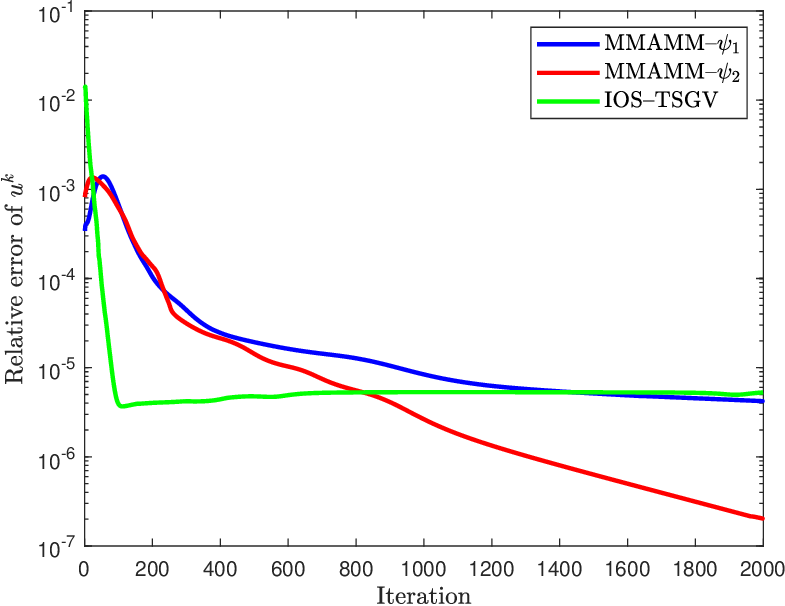}
    }
\caption{Evolution of the relative error of $\mathbf{u}^k$ for IOS--TSGV and the proposed MMAMM algorithms on TestImg$4$ and TestImg$8$ with different noise levels.} \label{figure_relative_error}
\end{figure}

Figure \ref{figure_relative_error} illustrates the evolution of the relative error of $\mathbf{u}^k$ for IOS--TSGV and the proposed MMAMM algorithms on TestImg4 and TestImg8. All algorithms are performed for 2000 iterations. The curves show that all the tested algorithms converge stably. Owing to the acceleration strategy, IOS--TSGV converges faster than the MMAMM algorithms. However, its final convergence value is larger than that of the proposed MMAMM methods.

\subsection{Comparison of state-of-the-art methods}
Table \ref{table_comparison} reports the quantitative comparisons among IOS--TSGV, HALM--EE, and the proposed MMAMM algorithms in terms of PSNR, SSIM, number of iterations, running time, and average time per iteration. The results in the table show that when the noise variance is $0.0025$, the proposed MMAMM--$\psi_1$ delivers competitive and often superior restoration quality compared with IOS--TSGV and HALM--EE in terms of PSNR and SSIM. When the noise variance increases to $0.005$, MMAMM--$\psi_1$ consistently yields the best PSNR and SSIM values, indicating improved robustness in the high-noise situation. Although the proposed MMAMM algorithms generally require more iterations than IOS--TSGV, their average time per iteration is lower, as shown by the Time/$\#$Iter statistics, indicating that each MMAMM iteration is computationally efficient. Compared with HALM--EE, the proposed methods not only achieve higher PSNR and SSIM values but also require significantly less runtime in most experiments. Figure \ref{figure_comparison} presents the restored images and the corresponding residual images for Gaussian denoising on TestImg5 and TestImg11. From the residual images, it is clear that the proposed MMAMM algorithms for the TSGV-based model preserve image edges more effectively than IOS--TSGV and HALM--EE.

\begin{table}[htbp]
\centering
\caption{Quantitative results of IOS--TSGV, HALM--EE and the proposed MMAMM algorithms, including the PSNR values, SSIM values, number of iterations ($\#$Iter), running times (Time) and average time per iteration (Time/$\#$Iter).}
\label{table_comparison}
\scriptsize{
\setlength{\tabcolsep}{0.1mm}
\begin{tabular}{l cccc cccc cccc cccc}
\toprule
\multirow{2}{*}{Image} & \multicolumn{4}{c}{IOS--TSGV} & \multicolumn{4}{c}{HALM--EE} & \multicolumn{4}{c}{MMAMM--$\psi_1$} & \multicolumn{4}{c}{MMAMM--$\psi_2$} \\
\cmidrule(lr){2-5} \cmidrule(lr){6-9} \cmidrule(lr){10-13} \cmidrule(lr){14-17}
 & PSNR & SSIM & $\#$Iter & Time & PSNR & SSIM & $\#$Iter & Time & PSNR & SSIM & $\#$Iter & Time & PSNR & SSIM & $\#$Iter & Time \\
\midrule
\multicolumn{17}{c}{Noise Level: 0.0025} \\
\midrule
TestImg4 & 35.13 & 0.9293 & 17 & 0.1 & 36.35 & 0.9431 & 500 & 0.3 & \textbf{36.87} & \textbf{0.9512} & 197 & 0.1 & 35.84 & 0.9250 & 151 & 0.1 \\
TestImg5 & 32.29 & \textbf{0.9164} & 26 & 0.2 & 31.70 & 0.9035 & 447 & 4.2 & \textbf{32.30} & 0.9158 & 240 & 1.1 & 32.02 & 0.9019 & 177 & 0.8 \\
TestImg6 & 31.60 & 0.8844 & 42 & 0.4 & 31.36 & 0.8819 & 500 & 6.3 & \textbf{31.61} & \textbf{0.8867} & 304 & 1.4 & 31.26 & 0.8591 & 237 & 1.1 \\
TestImg7 & \textbf{34.23} & \textbf{0.9061} & 22 & 0.9 & 33.62 & 0.8858 & 144 & 4.5 & 34.09 & 0.8998 & 234 & 4.8 & 33.78 & 0.8871 & 184 & 3.6 \\
\midrule
\multicolumn{17}{c}{Noise Level: 0.005} \\
\midrule
TestImg8 & 33.73 & 0.9000 & 42 & 0.1 & 34.17 & 0.9200 & 500 & 0.2 & \textbf{35.74} & \textbf{0.9498} & 204 & 0.1 & 34.65 & 0.9120 & 222 & 0.1 \\
TestImg9 & 32.05 & 0.8458 & 20 & 0.2 & 31.69 & 0.8514 & 500 & 6.4 & \textbf{32.54} & \textbf{0.8536} & 300 & 1.4 & 31.80 & 0.8302 & 223 & 1.0 \\
TestImg10 & 28.40 & 0.7749 & 31 & 0.3 & 28.28 & 0.7687 & 500 & 4.8 & \textbf{28.44} & \textbf{0.7753} & 403 & 1.9 & 28.14 & 0.7619 & 339 & 1.6 \\
TestImg11 & 29.87 & 0.8879 & 19 & 0.2 & 29.20 & 0.8716 & 500 & 5.4 & \textbf{30.42} & \textbf{0.8953} & 355 & 2.0 & 29.98 & 0.8780 & 287 & 1.6 \\
\midrule
Time/$\#$Iter & \multicolumn{4}{c}{0.012} & \multicolumn{4}{c}{0.011} & \multicolumn{4}{c}{0.006} & \multicolumn{4}{c}{0.006} \\
\bottomrule
\end{tabular}
}
\end{table}

\begin{figure}[t]
    \centering
    \subfigure[Noisy image]{
        \begin{minipage}[b]{0.17\linewidth}
        \centering
        \includegraphics[scale=0.4]{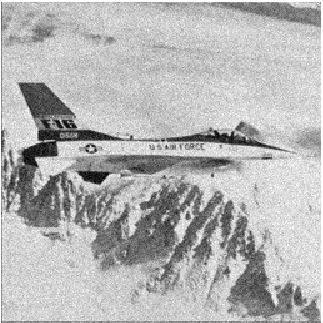}\\
        \vspace{0.02cm}
        \includegraphics[scale=0.2]{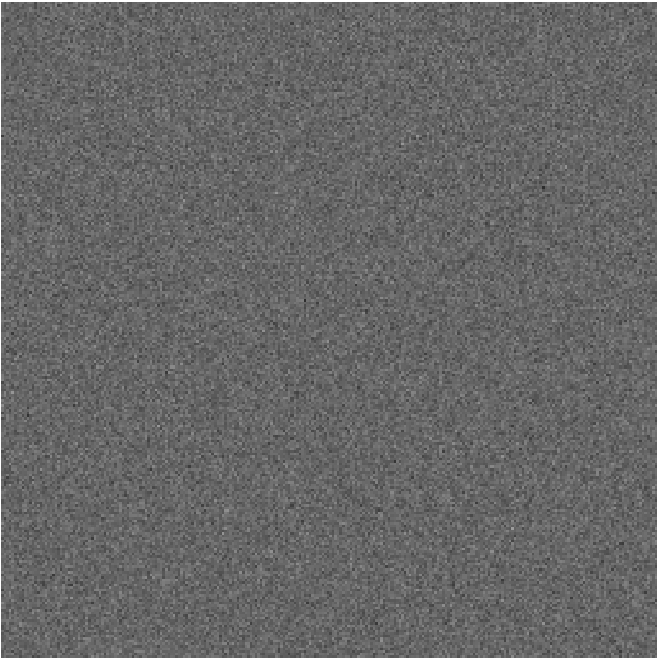}\\
        \vspace{0.02cm}
        \includegraphics[scale=0.315]{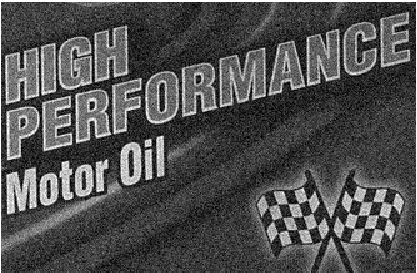}\\
        \vspace{0.02cm}
        \includegraphics[scale=0.15]{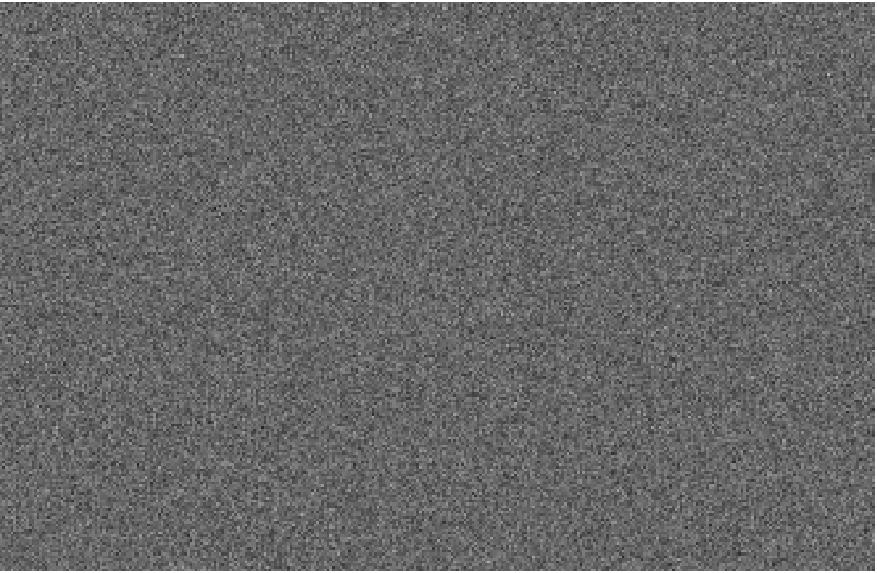}\\
        \vspace{0.02cm}
        \end{minipage}
        }
    \subfigure[IOS--TSGV]{
        \begin{minipage}[b]{0.17\linewidth}
        \centering
        \includegraphics[scale=0.2]{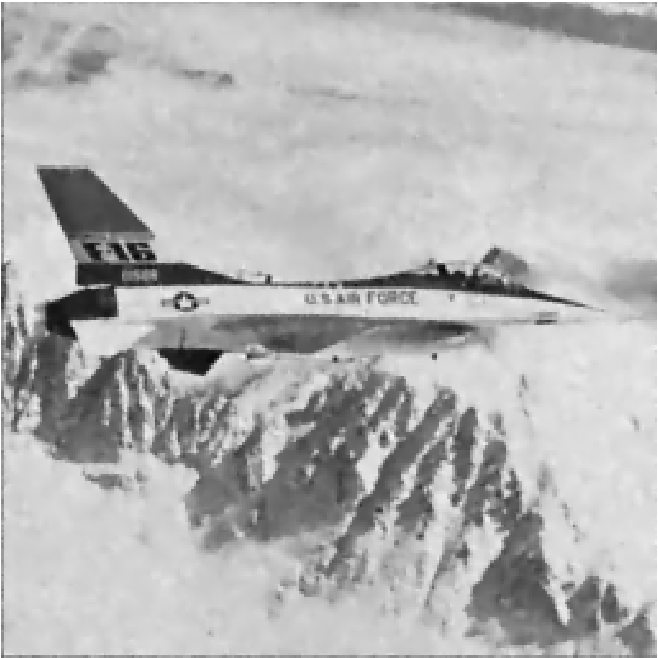}\\
        \vspace{0.02cm}
        \includegraphics[scale=0.2]{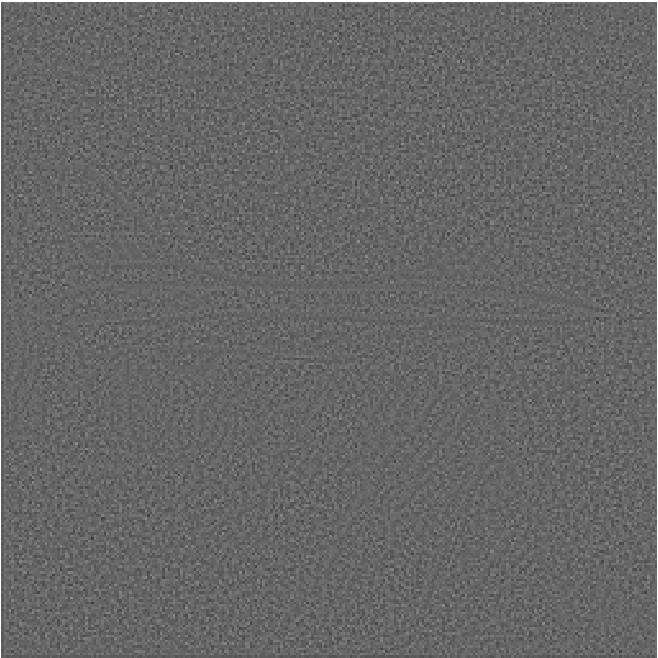}\\
        \vspace{0.02cm}
        \includegraphics[scale=0.15]{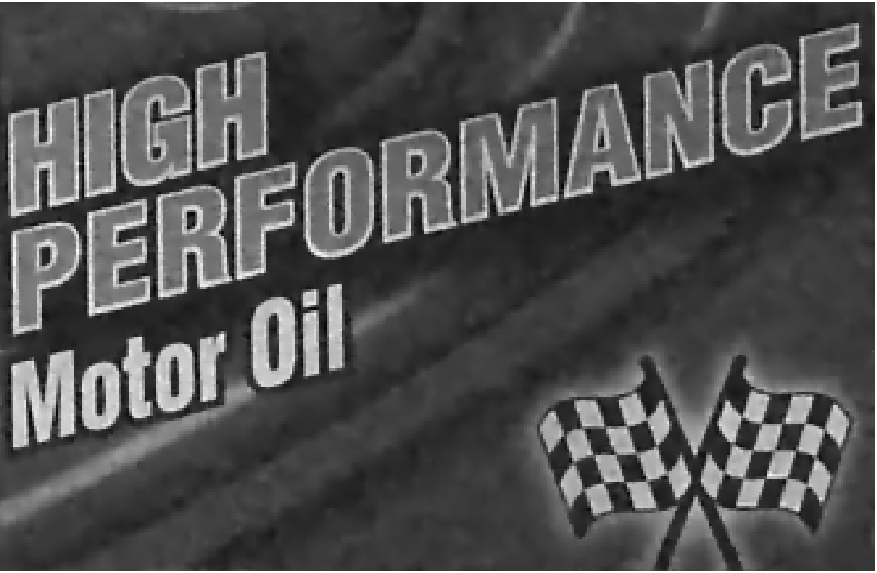}\\
        \vspace{0.02cm}
        \includegraphics[scale=0.15]{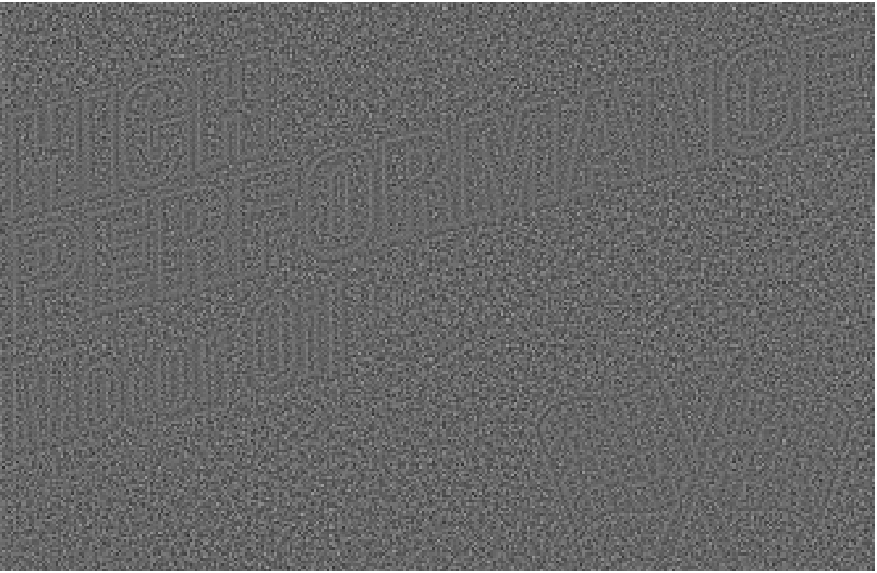}\\
        \vspace{0.02cm}
        \end{minipage}
        }
    \subfigure[HALM--EE]{
        \begin{minipage}[b]{0.17\linewidth}
        \centering
        \includegraphics[scale=0.2]{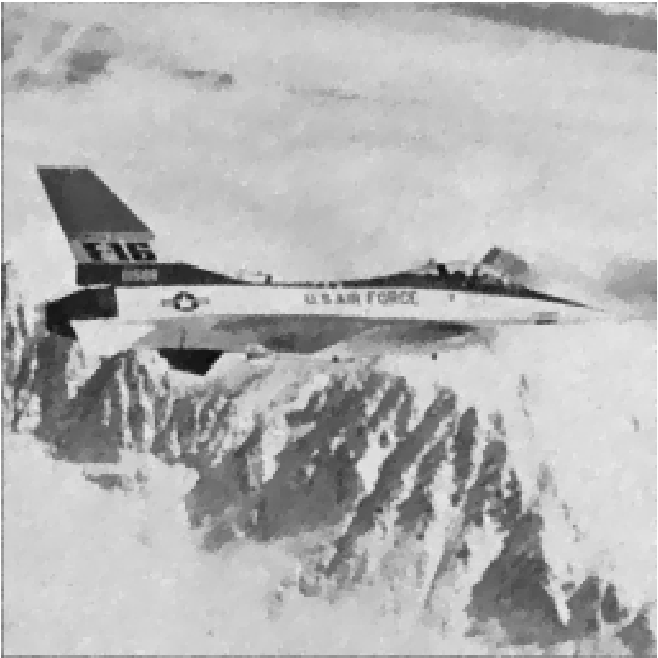}\\
        \vspace{0.02cm}
        \includegraphics[scale=0.2]{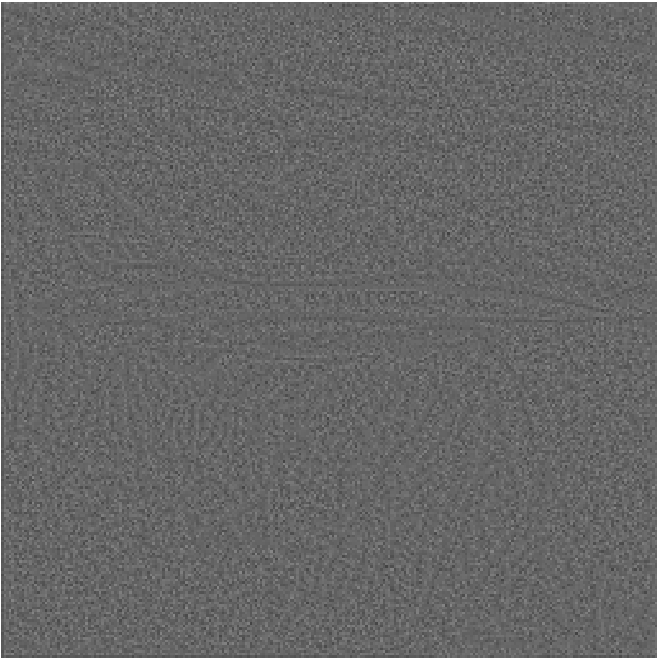}\\
        \vspace{0.02cm}
        \includegraphics[scale=0.15]{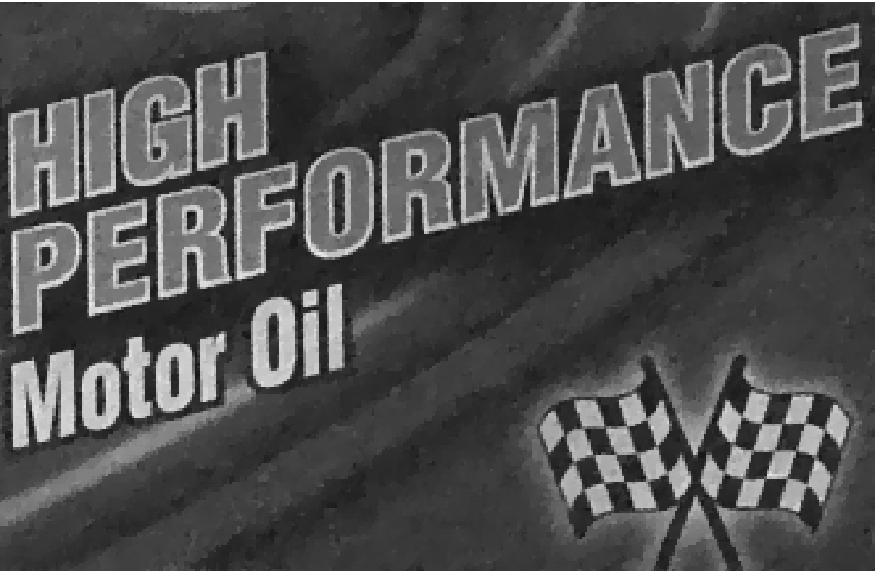}\\
        \vspace{0.02cm}
        \includegraphics[scale=0.15]{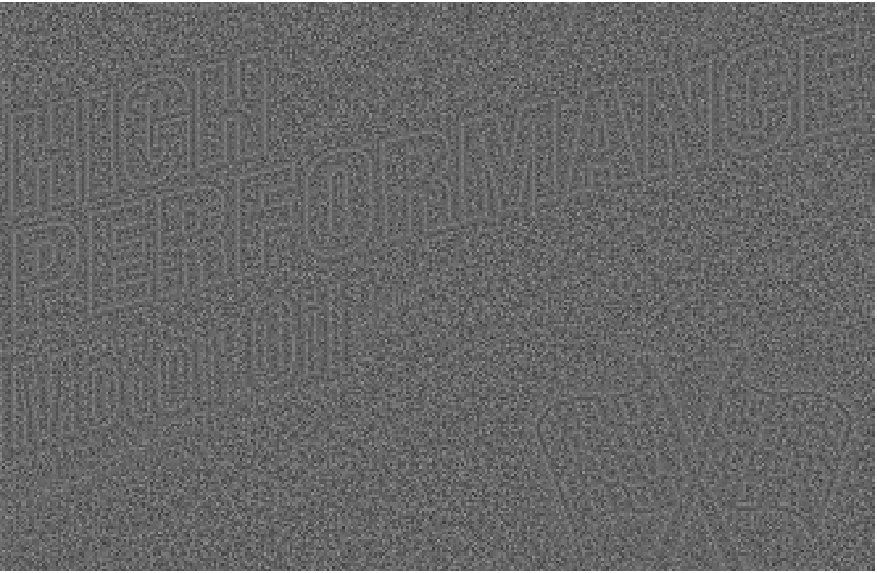}\\
        \vspace{0.02cm}
        \end{minipage}
        }
    \subfigure[MMAMM--$\psi_1$]{
        \begin{minipage}[b]{0.17\linewidth}
        \centering
        \includegraphics[scale=0.2]{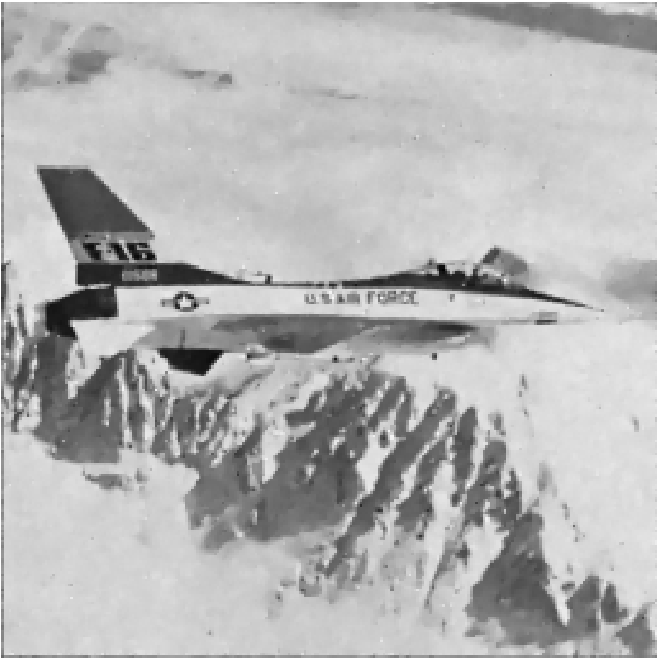}\\
        \vspace{0.02cm}
        \includegraphics[scale=0.2]{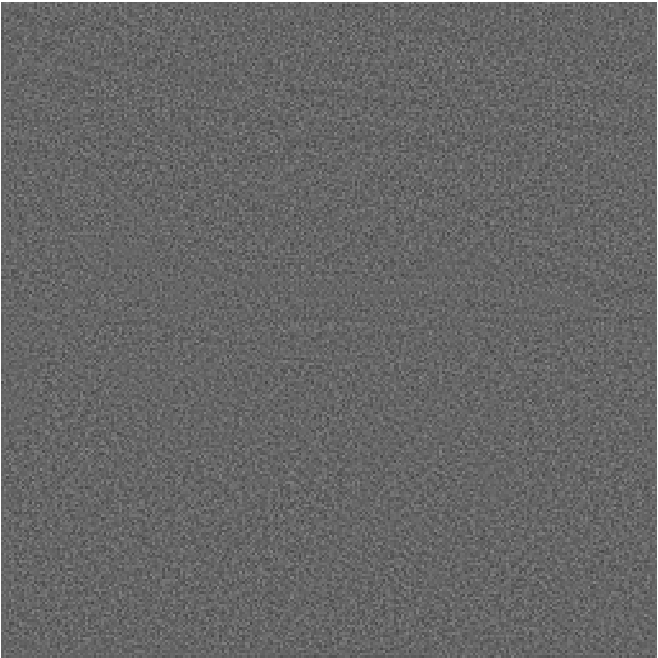}\\
        \vspace{0.02cm}
        \includegraphics[scale=0.15]{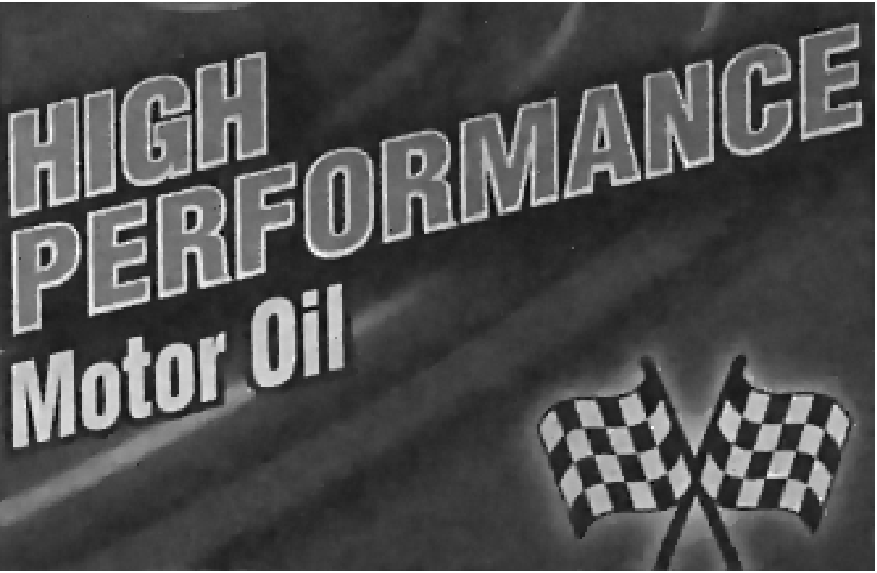}\\
        \vspace{0.02cm}
        \includegraphics[scale=0.15]{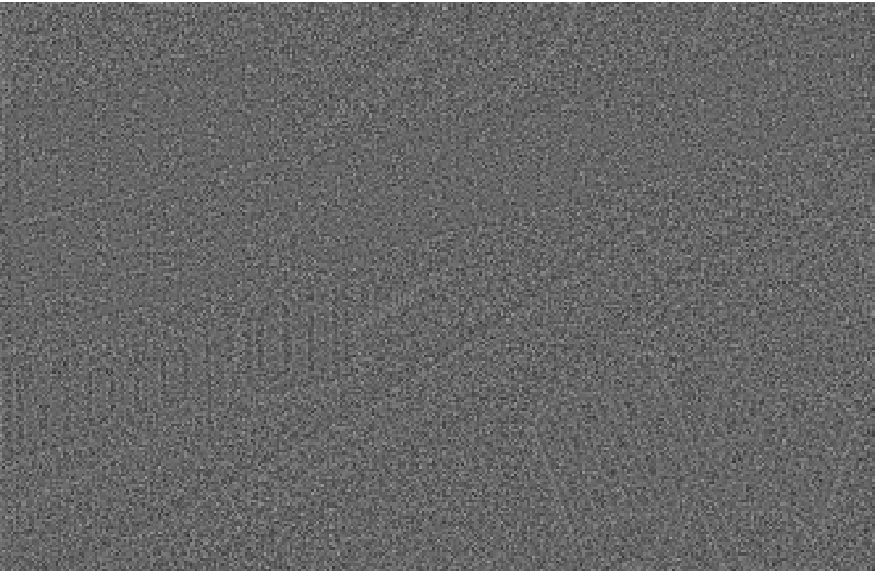}\\
        \vspace{0.02cm}
        \end{minipage}
        }
    \subfigure[MMAMM--$\psi_2$]{
        \begin{minipage}[b]{0.17\linewidth}
        \centering
        \includegraphics[scale=0.2]{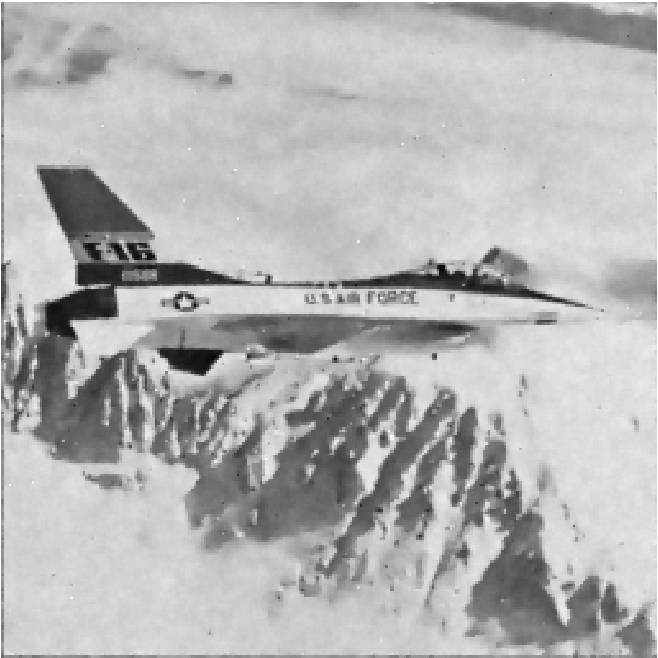}\\
        \vspace{0.02cm}
        \includegraphics[scale=0.2]{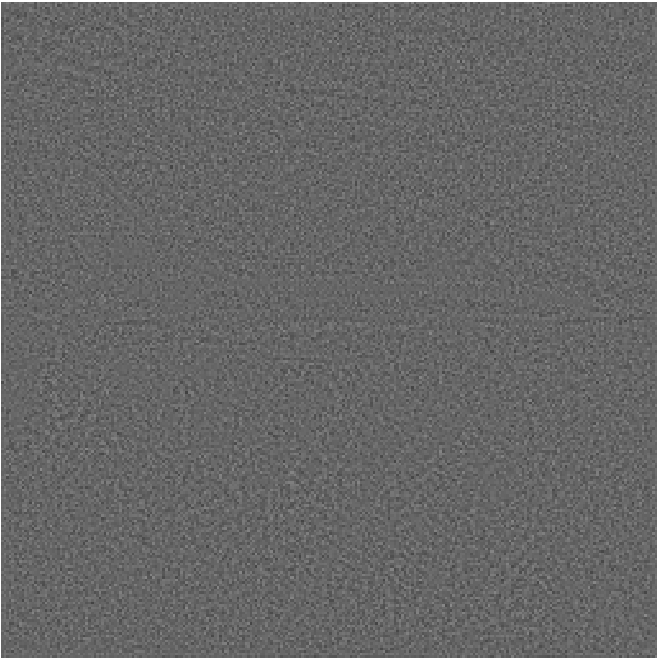}\\
        \vspace{0.02cm}
        \includegraphics[scale=0.15]{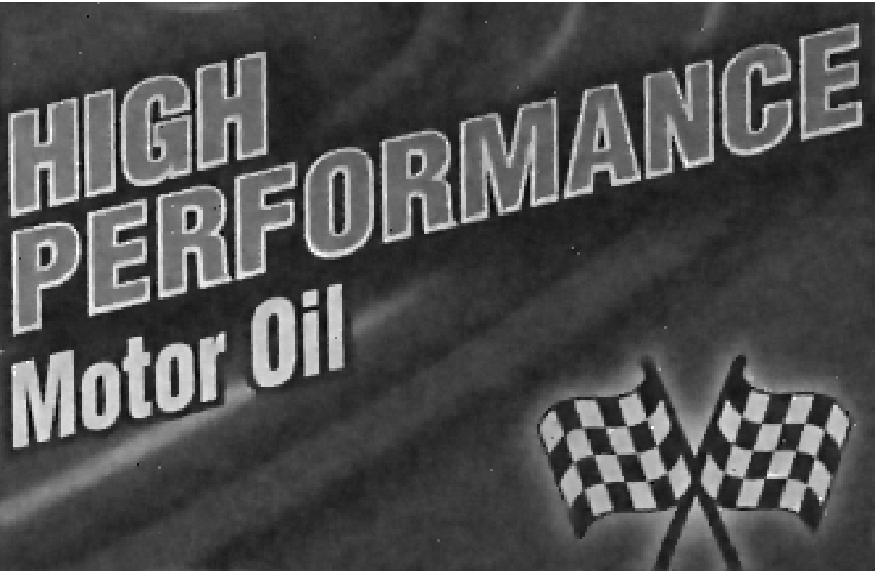}\\
        \vspace{0.02cm}
        \includegraphics[scale=0.15]{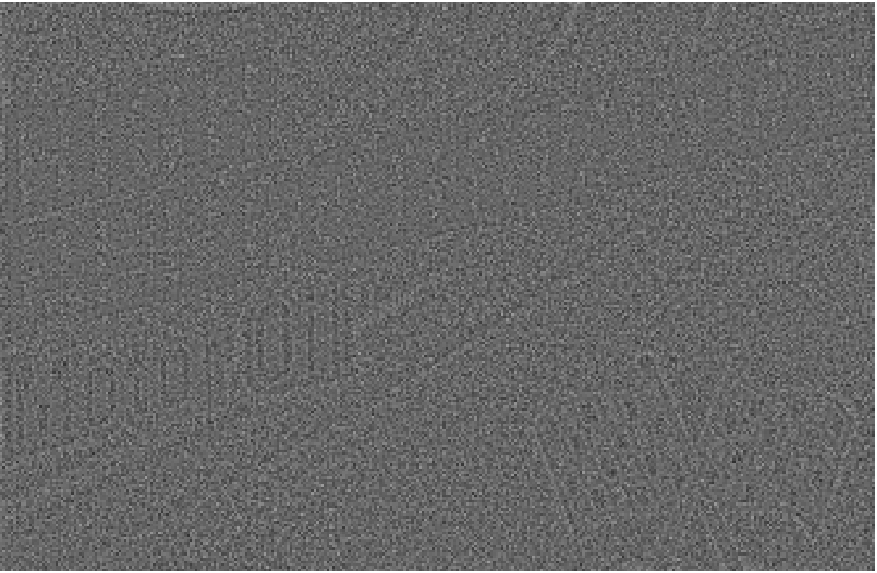}\\
        \vspace{0.02cm}
        \end{minipage}
        }
\caption{Restored images and residual images for Gaussian denoising on TestImg5 and TestImg11. The first and third rows show the restored images, while the second and fourth rows show the corresponding residual images. The residual images are computed by $\mathbf{f}-\mathbf{u}+0.4$.} \label{figure_comparison}
\end{figure}

\subsection{Confocal non-line-of-sight imaging}
NLOS imaging aims to recover the geometry or reflectivity of objects hidden from direct line of sight by analyzing time-resolved light-transport measurements. In an NLOS imaging system, a laser pulse emitted from the laser source first illuminates a point (referred to as the illumination point) on the visible wall. After reflecting off the wall, part of the light propagates toward the hidden object outside the direct line of sight. The hidden object then reflects the light to the visible wall, where the reflected signal is finally captured by the detector at another point (the detection point) on the visible wall. This process corresponds to a multi-bounce indirect light transport procedure and forms the physical foundation of NLOS imaging.

When the illumination points and the detection points are located at different positions on the visible wall, the imaging configuration is referred to as a non-confocal NLOS imaging system. In contrast, when the illumination points and the detection points coincide, the system is called a confocal NLOS imaging system. The confocal NLOS imaging system significantly simplifies the light transport model and reduces the computational complexity of the reconstruction problem. In this subsection, we apply the proposed method to the confocal NLOS imaging problem.

Let $(x,y,z)\in\Omega\subset\mathbb{R}^3$ denote the three-dimensional scene coordinates of the hidden volume. The visible wall is defined as the plane \(
\{(x,y,z)\in\mathbb{R}^3:z=0\}
\).
Let $(x', y',0)$ denote the illumination point and detection point on the visible wall. Under the confocal configuration, the continuous volumetric albedo model can be expressed as \cite{ding2024curvature,o2018confocal}
\[
\tau(x',y',t)
=
\iiint_{\Omega}
\frac{u(x,y,z)}{r^4}
\delta\!\left(2r-ct\right)\mathrm{d}x\mathrm{d}y\mathrm{d}z,
\]
where
\[
r=\sqrt{(x-x')^2+(y-y')^2+z^2},
\]
$u(x,y,z)$ denotes the volumetric albedo of the hidden scene, $c$ is the speed of light, and $\delta(\cdot)$ represents the Dirac delta function enforcing the time-of-flight constraint.

After discretization, the corresponding discrete formation model is given by \cite{o2018confocal}
\[
\boldsymbol{\tau}=\mathbf{A}\mathbf{u}
=
\mathbf{R}_t^{-1}\mathbf{H}\mathbf{R}_z\mathbf{u},
\]
where $\mathbf{u}$ and $\boldsymbol{\tau}$ denote the discretized hidden volume and transient measurements, respectively. The matrix $\mathbf{A}$ represents the light transport operator, $\mathbf{H}$ denotes a shift-invariant three-dimensional convolution operator associated with the confocal light propagation model, and $\mathbf{R}_t$ and $\mathbf{R}_z$ are transformation operators acting on the temporal and spatial dimensions, respectively. This formulation converts the NLOS reconstruction problem into a deconvolution problem in the transformed domain, enabling efficient numerical implementation using FFT-based algorithms \cite{o2018confocal}.

Compared with Gaussian denoising for grayscale images, the major modification of the proposed MMAMM algorithm when applied to the NLOS imaging problem, beyond the increase in dimensionality, lies in how the data-fidelity term is formulated. In the NLOS reconstruction task, the fidelity term is given by
\[
\frac{\lambda}{2}\|\mathbf{S}\mathbf{A}\mathbf{u}-\boldsymbol{\tau}\|_2^2,
\]
where $\mathbf{S}$ is the under-sampling operator. At iteration $k+1$, we employ a second-order surrogate approximation of the fidelity term:
\begin{align*}
\frac{\lambda}{2}\|\mathbf{S}\mathbf{A}\mathbf{u}-\boldsymbol{\tau}\|_2^2
&\approx
\frac{\lambda}{2}\|\mathbf{S}\mathbf{A}\mathbf{u}^k-\boldsymbol{\tau}\|_2^2
+
\lambda\left\langle
\mathbf{A}^\top \mathbf{S}^\top(\mathbf{S}\mathbf{A}\mathbf{u}^k-\boldsymbol{\tau}),
\mathbf{u}-\mathbf{u}^k
\right\rangle
+
\frac{L}{2}\|\mathbf{u}-\mathbf{u}^k\|_2^2\\
&= \frac{L}{2} \left\|\mathbf{u} - \left(\mathbf{u}^k-\frac{\lambda}{L}\mathbf{A}^\top \mathbf{S}^\top
\left(
\mathbf{S}\mathbf{A}\mathbf{u}^k-\boldsymbol{\tau}\right)\right)\right\|_2^2
+\Tilde{C},
\end{align*}
where $L=\lambda\|\mathbf{S}\mathbf{A}\|_2^2$, $\Tilde{C}$ is a constant independent of $\mathbf{u}$. Consequently, the resulting approximation with respect to $\mathbf{u}$ becomes a convex quadratic functional, which admits an efficient closed-form solution in the Fourier domain.

For NLOS imaging experiments, unless otherwise specified, we set the maximum number of iterations to $300$, initialize the algorithm with $\mathbf{u}^0=0$ in our methods, and use the same remaining parameters as in the image denoising experiments.

\begin{table}[tbp]
\centering
\footnotesize{
\caption{Quantitative results of LCT, CNLOS, and the proposed MMAMM, including the PSNR values and SSIM values.} \label{table_nlos_comparison}
\setlength{\tabcolsep}{1.5mm}{
\vspace{0.1cm}
\begin{tabular}{cccccccccc}
\toprule
                         &                 & \multicolumn{2}{c}{LCT} & \multicolumn{2}{c}{CNLOS} & \multicolumn{2}{c}{MMAMM--$\psi_1$} & \multicolumn{2}{c}{MMAMM--$\psi_2$} \\
                         \cmidrule(lr){3-4} \cmidrule(lr){5-6}
                         \cmidrule(lr){7-8} \cmidrule(lr){9-10}
                         & scanning points & PSNR       & SSIM       & PSNR            & SSIM    & PSNR                 & SSIM         & PSNR        & SSIM                  \\ \midrule
\multirow{4}{*}{bowling} & $64\times64$    & 14.69      & 0.2500     & 16.62           & 0.4703  & \textbf{16.82}       & 0.4743       & 16.81       & \textbf{0.4744}       \\
                         & $32\times32$    & 14.63      & 0.2470     & 16.54           & 0.4250  & \textbf{16.84}       & 0.4642       & 16.83       & \textbf{0.4647}       \\
                         & $16\times16$    & 14.30      & 0.2330     & 16.16           & 0.3277  & \textbf{16.57}       & 0.4225       & 16.48       & \textbf{0.4352}       \\
                         & $8\times8$      & 12.88      & 0.1798     & \textbf{15.68}  & 0.2505  & 15.26                & 0.3852       & 15.21       & \textbf{0.3879}       \\ \bottomrule
\end{tabular}}}
\end{table}
\begin{figure}[tbp]
\centering


\begin{minipage}[c]{0.02\linewidth}
\centering \rotatebox{90}{$64\times64$}
\end{minipage}
\begin{minipage}[c]{0.2\linewidth}
\centering
\includegraphics[scale=0.38]{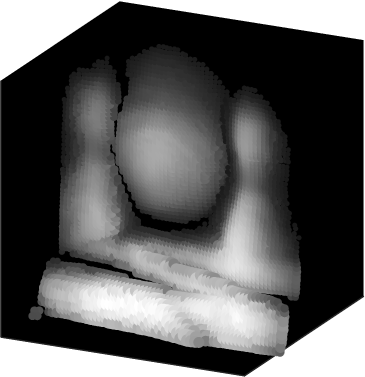}
\end{minipage}
\begin{minipage}[c]{0.2\linewidth}
\centering
\includegraphics[scale=0.38]{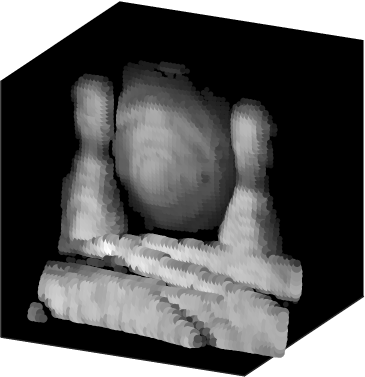}
\end{minipage}
\begin{minipage}[c]{0.2\linewidth}
\centering
\includegraphics[scale=0.38]{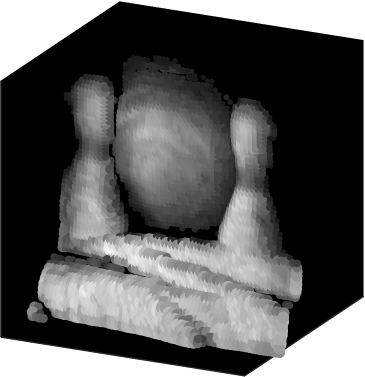}
\end{minipage}
\begin{minipage}[c]{0.2\linewidth}
\centering
\includegraphics[scale=0.38]{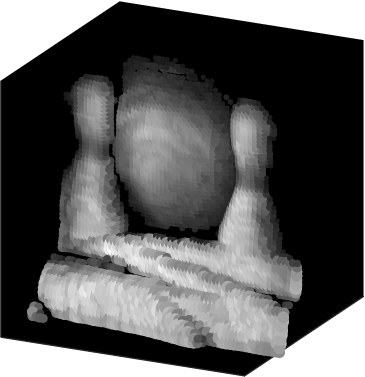}
\end{minipage}
\vspace{0.2cm}


\begin{minipage}[c]{0.02\linewidth}
\centering \rotatebox{90}{$32\times32$}
\end{minipage}
\begin{minipage}[c]{0.2\linewidth}
\centering
\includegraphics[scale=0.38]{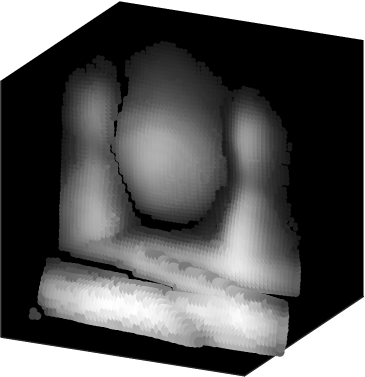}
\end{minipage}
\begin{minipage}[c]{0.2\linewidth}
\centering
\includegraphics[scale=0.38]{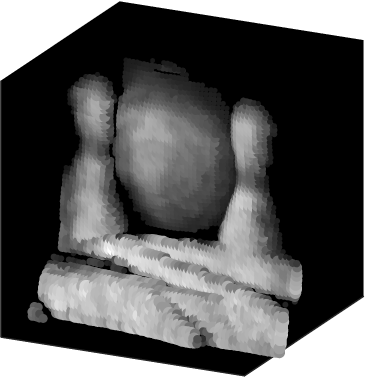}
\end{minipage}
\begin{minipage}[c]{0.2\linewidth}
\centering
\includegraphics[scale=0.38]{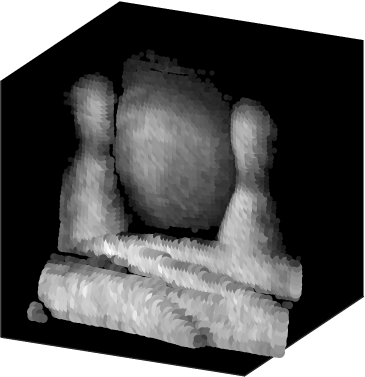}
\end{minipage}
\begin{minipage}[c]{0.2\linewidth}
\centering
\includegraphics[scale=0.38]{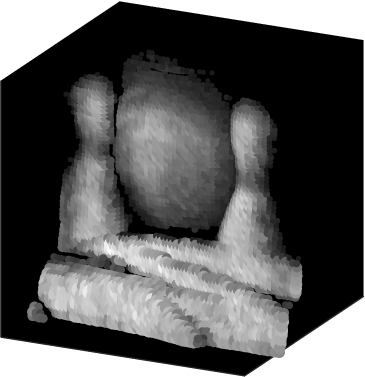}
\end{minipage}
\vspace{0.15cm}


\begin{minipage}[c]{0.02\linewidth}
\centering \rotatebox{90}{$16\times16$}
\end{minipage}
\begin{minipage}[c]{0.2\linewidth}
\centering
\includegraphics[scale=0.38]{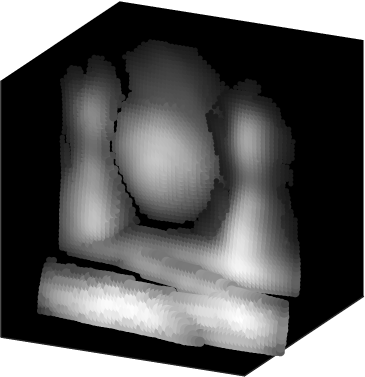}
\end{minipage}
\begin{minipage}[c]{0.2\linewidth}
\centering
\includegraphics[scale=0.38]{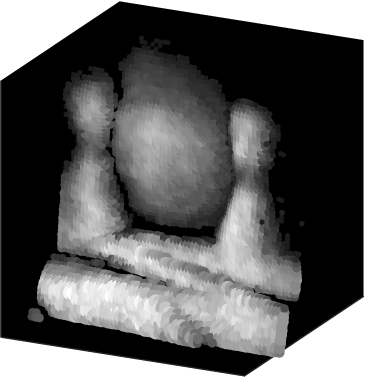}
\end{minipage}
\begin{minipage}[c]{0.2\linewidth}
\centering
\includegraphics[scale=0.38]{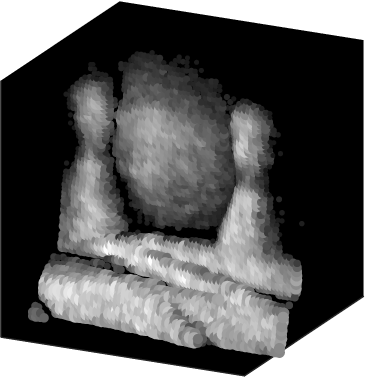}
\end{minipage}
\begin{minipage}[c]{0.2\linewidth}
\centering
\includegraphics[scale=0.38]{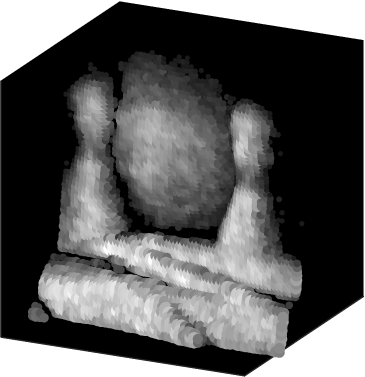}
\end{minipage}
\vspace{0.15cm}


\begin{minipage}[c]{0.02\linewidth}
\centering \rotatebox{90}{$8\times8$}
\end{minipage}
\begin{minipage}[c]{0.2\linewidth}
\centering
\includegraphics[scale=0.38]{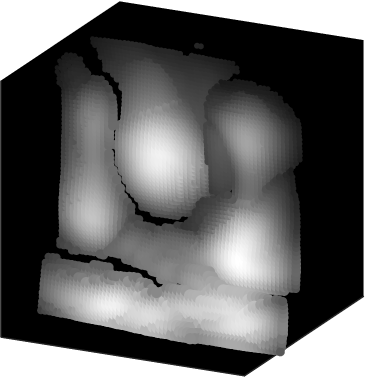}
\end{minipage}
\begin{minipage}[c]{0.2\linewidth}
\centering
\includegraphics[scale=0.38]{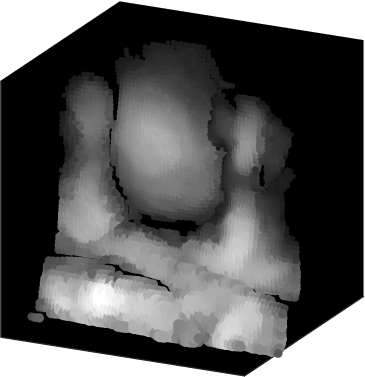}
\end{minipage}
\begin{minipage}[c]{0.2\linewidth}
\centering
\includegraphics[scale=0.38]{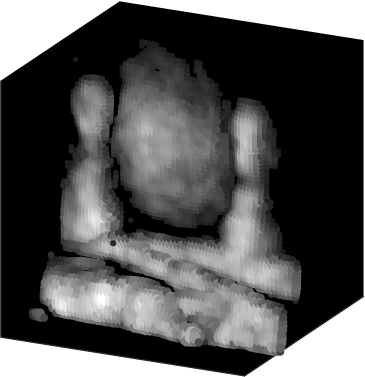}
\end{minipage}
\begin{minipage}[c]{0.2\linewidth}
\centering
\includegraphics[scale=0.38]{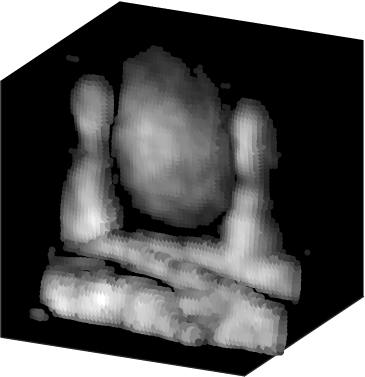}
\end{minipage}
\vspace{0.15cm}


\begin{minipage}[c]{0.02\linewidth}
\end{minipage}
\begin{minipage}[c]{0.22\linewidth}
\centering LCT
\end{minipage}
\begin{minipage}[c]{0.22\linewidth}
\centering CNLOS
\end{minipage}
\begin{minipage}[c]{0.22\linewidth}
\centering MMAMM--$\psi_1$
\end{minipage}
\begin{minipage}[c]{0.22\linewidth}
\centering MMAMM--$\psi_2$
\end{minipage}


\caption{Restored images for NLOS imaging using different reconstruction methods. Rows correspond to different numbers of scanning points ($64\times64$, $32\times32$, $16\times16$ and $8\times8$),
while columns correspond to different reconstruction methods.}
\label{figure_nlos_bowling}
\end{figure}

We then evaluate the performance of the proposed methods on NLOS imaging data by comparing them with several state-of-the-art reconstruction approaches, including the light-cone transform (LCT) \cite{o2018confocal} and the curvature regularization method for NLOS imaging (CNLOS) \cite{ding2024curvature}, with the maximum number of iterations set to 300. We first conduct experiments on the bowling scene \cite{ye2021compressed}, where $64\times64$ scanning points are used to cover a $1\times1\mathrm{m}^2$ square region on the visible wall with a temporal resolution of $256$ bins and a bin width of $32$ ps.
To evaluate the robustness of the proposed methods under different sampling conditions, we compare the reconstruction results using fully sampled measurements with $64\times64$ scanning points and under-sampled measurements with $32\times32$, $16\times16$, and $8\times8$ scanning points.

Table \ref{table_nlos_comparison} presents the quantitative comparisons among LCT, CNLOS, and the proposed MMAMM in terms of PSNR and SSIM. The proposed MMAMM generally achieves superior reconstruction performance compared with LCT and CNLOS. In particular, for the cases with $64\times64$, $32\times32$, and $16\times16$ scanning points, MMAMM--$\psi_1$ attains the highest PSNR values, while MMAMM--$\psi_2$ achieves the best SSIM values. In the extremely sparse sampling case with only $8\times8$ scanning points, CNLOS achieves the highest PSNR, whereas the proposed MMAMM yields significantly higher SSIM. Overall, the quantitative results demonstrate the robustness and effectiveness of the proposed MMAMM for NLOS imaging under different scanning points. Visual comparisons are provided in Figure \ref{figure_nlos_bowling}.

Next, we evaluate the Stanford Bunny data from the Zaragoza NLOS synthetic dataset\footnote{\url{https://graphics.unizar.es/nlos_dataset.html}}. The full set of $64\times64$ scanning points covers a $0.6\times0.6 \mathrm{m}^2$ square region on the visible wall. The captured measurements contain 512 temporal bins, where the photon propagation distance corresponding to each bin is $0.0025 \mathrm{m}$. We compare the reconstruction results using the fully sampled measurements with $64\times64$ scanning points and the under-sampled measurements with $32\times32$ scanning points. In addition, the maximum number of iterations of our methods is set to $100$.

The comparison results are displayed in Figure \ref{figure_nlos_bunny}. It can be observed that all methods recover the main structure of the Bunny under both $64\times64$ and $32\times32$ scanning points. As the number of scanning points decreases, the reconstruction quality of all methods deteriorates to some extent, resulting in blurred boundaries and a loss of fine details.


\begin{figure}[tbp]
\centering


\begin{minipage}[c]{0.02\linewidth}
\centering \rotatebox{90}{$64\times64$}
\end{minipage}
\begin{minipage}[c]{0.2\linewidth}
\centering
\includegraphics[scale=0.3]{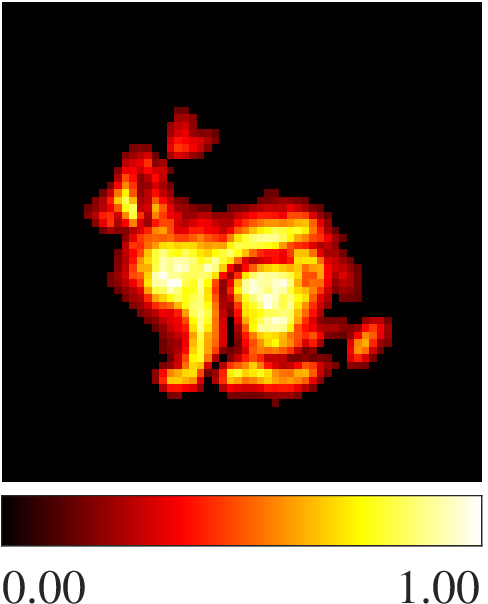}
\end{minipage}
\begin{minipage}[c]{0.2\linewidth}
\centering
\includegraphics[scale=0.3]{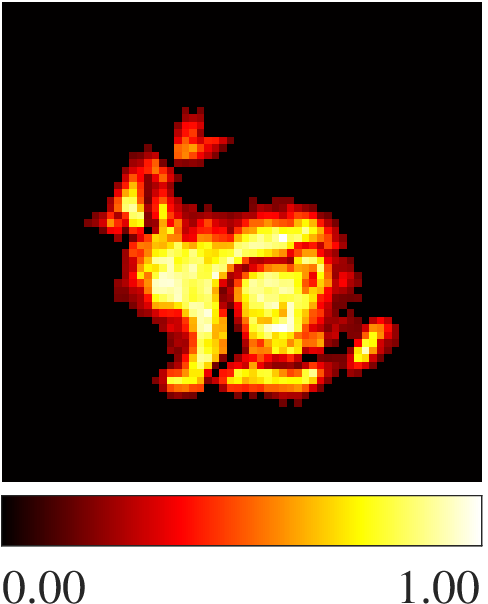}
\end{minipage}
\begin{minipage}[c]{0.2\linewidth}
\centering
\includegraphics[scale=0.3]{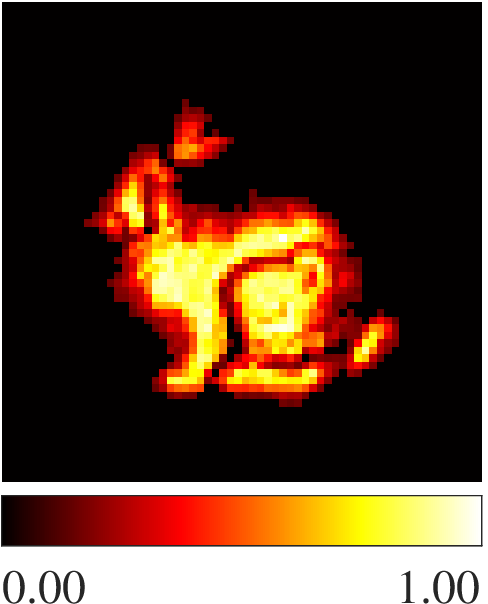}
\end{minipage}
\begin{minipage}[c]{0.2\linewidth}
\centering
\includegraphics[scale=0.3]{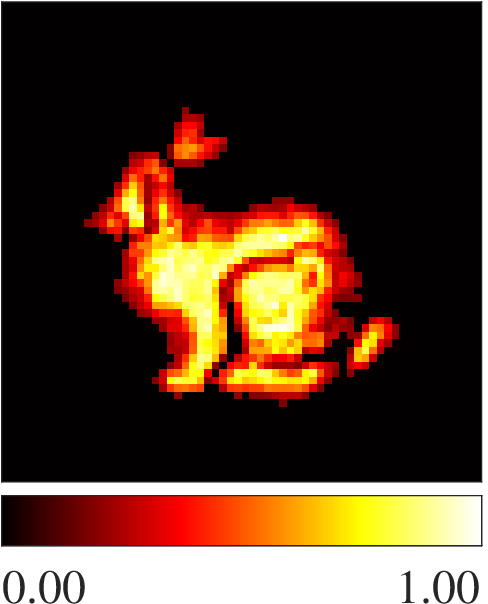}
\end{minipage}
\vspace{0.15cm}


\begin{minipage}[c]{0.02\linewidth}
\centering \rotatebox{90}{$32\times32$}
\end{minipage}
\begin{minipage}[c]{0.2\linewidth}
\centering
\includegraphics[scale=0.3]{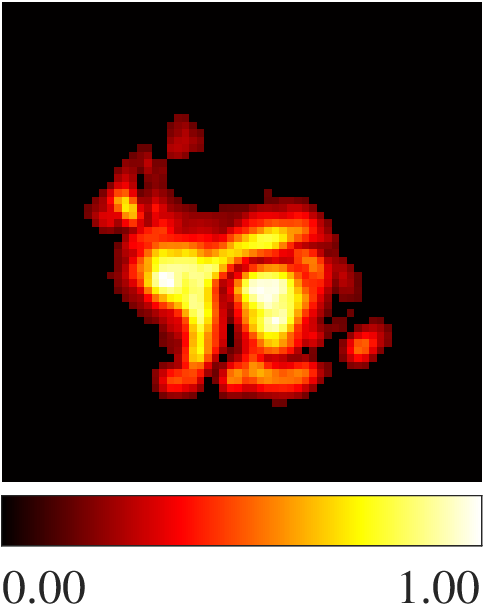}
\end{minipage}
\begin{minipage}[c]{0.2\linewidth}
\centering
\includegraphics[scale=0.3]{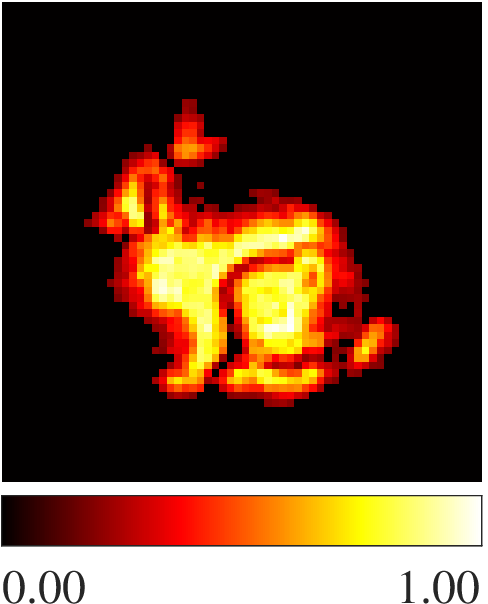}
\end{minipage}
\begin{minipage}[c]{0.2\linewidth}
\centering
\includegraphics[scale=0.3]{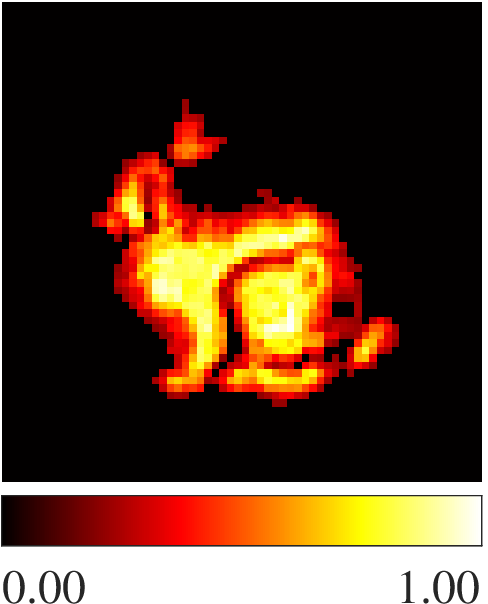}
\end{minipage}
\begin{minipage}[c]{0.2\linewidth}
\centering
\includegraphics[scale=0.3]{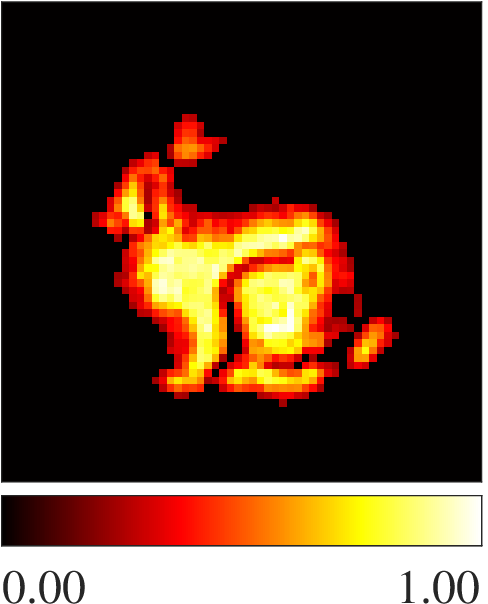}
\end{minipage}
\vspace{0.15cm}

\begin{minipage}[c]{0.02\linewidth}
\end{minipage}
\begin{minipage}[c]{0.22\linewidth}
\centering LCT
\end{minipage}
\begin{minipage}[c]{0.22\linewidth}
\centering CNLOS
\end{minipage}
\begin{minipage}[c]{0.22\linewidth}
\centering MMAMM--$\psi_1$
\end{minipage}
\begin{minipage}[c]{0.22\linewidth}
\centering MMAMM--$\psi_2$
\end{minipage}


\caption{Restored images for NLOS imaging using different reconstruction methods. Rows correspond to different numbers of scanning points ($64\times64$ and $32\times32$), while columns correspond to different reconstruction methods.}
\label{figure_nlos_bunny}
\end{figure}

%% file: sections/sec-conclusion.tex
\section{Conclusions} \label{sec:conclusion}

In this paper, we developed a novel bilinear decomposition method for a class of nonconvex and nonlinear TSGV variational models. The reformulation preserves the geometric structure of the original regularizer while making the numerical solution more tractable. We also provided a geometric interpretation of the constraints induced by different scaling functions.  The asymptotic analysis near large-gradient regions indicates that the cone constraint is better adapted to edges and corners. This theoretical observation is consistent with the numerical results.

For the resulting nonconvex constrained optimization problem, we proposed an algorithm based on alternating minimization and an MM strategy. The MM construction yields surrogate subproblems with improved tractability and ensures a monotone decrease of the energy. A dedicated numerical scheme was developed for the cone-constrained subproblems. The convergence analysis provides theoretical justification for the proposed scheme. Numerical experiments on image denoising and NLOS imaging demonstrate the robustness and flexibility of the method. Future work includes extending the proposed framework to other nonlinear geometric regularizers, developing more efficient solvers for three-dimensional imaging, and investigating adaptive choices of scaling functions for different image structures and inverse problems.